\documentclass[final,authoryear,5p,times]{elsarticle}

\usepackage{amsmath}
\usepackage{amssymb}
\usepackage{mathtools}
\usepackage{amsthm}
\usepackage{booktabs}
\usepackage{array}
\usepackage{siunitx}
\usepackage{subcaption}
\usepackage{tikz}
\usetikzlibrary{arrows.meta, calc, positioning}
\usepackage[protrusion=true,expansion=true]{microtype}
\usepackage[hidelinks]{hyperref}

\newcommand{\R}{\mathbb{R}}
\DeclareMathOperator{\diag}{diag}

\newcommand{\ind}{\mathbb{I}}
\newcommand{\pp}{\,\mathrm{pp}}
\theoremstyle{plain}

\newtheorem{corollary}{Corollary}
\newtheorem{proposition}{Proposition}
\theoremstyle{definition}
\newtheorem{definition}{Definition}

\graphicspath{{./Figures/}}
\journal{Control Engineering Practice}

\begin{document}

\begin{frontmatter}

\title{Real-Time Rulebook-Aware Nonlinear MPC for Autonomous Driving\\ with Priority-Biased Tiered Slacks\tnoteref{disclaimer}}
\tnotetext[disclaimer]{This paper represents research activities conducted within TORC Robotics LLC. It does not describe any production system, does not constitute the safety case for any Daimler Truck AG product or development program, and has not been reviewed by Daimler Truck AG for regulatory compliance or product liability purposes. The research was conducted using public benchmark data and does not reflect the proprietary safety architecture of any deployed or in-development TORC Robotics or Daimler Truck product.}
\author[inst1]{Hadi Hajieghrary}
\ead{Hadi.Hajieghrary@Torc.ai}
\author[inst1]{Benedikt Walter}
\ead{Benedikt.Walter@Torc.ai}
\author[inst1]{Chaitanya Shinde}
\ead{Chaitanya.Shinde@Torc.ai}
\author[inst2]{Paul Schmitt}
\ead{pauls@massrobotics.org}
\author[inst1]{Miguel Hurtado}
\ead{Miguel.Hurtado@Torc.ai}

\affiliation[inst1]{organization={TORC Robotics LLC, an independent subsidiary of Daimler Truck AG},
            country={USA}}
\affiliation[inst2]{organization={Reynolds \& Moore and MassRobotics},
            country={USA}}

\begin{abstract}
Autonomous-vehicle motion planners must resolve conflicts among safety, regulation, comfort, and efficiency in real time while exposing those decisions for audit. We present W-SQP, a weighted tiered-slack nonlinear model predictive controller (NMPC) that compiles nine driving-rule families into a four-tier shared-slack nonlinear program solved online with CasADi and IPOPT; the name denotes the weighted quadratic slack penalty, not a sequential-quadratic-programming solver. Strongly separated tier penalties bias residual violations toward lower-priority rules while leaving actuation bounds hard. The controller replans from its executed state at $10$\,Hz and records per-rule residuals on every cycle. A $90$\,ms solver-time limit returns an anytime iterate that is projected through the vehicle dynamics before execution; median and maximum observed wall-clock solve times were $28$ and $104$\,ms. We evaluate W-SQP in closed loop on 150 Waymo Open Motion Dataset scenarios in Waymax against reactive and proposal-and-select baselines, and introduce a log-independent protocol that separates safety and regulatory compliance from resemblance to the recorded human trajectory. Under this protocol, W-SQP shows no systematic group-level deficit relative to expert replay on the log-independent safety and regulatory rules, with several localized regressions in the hardest, highest-divergence scenarios. The results characterize W-SQP as an auditable, priority-biased, anytime-capable NMPC prototype rather than a hard-real-time or formally safe controller.
\end{abstract}


\end{frontmatter}

\section{Introduction}
\label{sec:intro}
We are interested in motion planners for autonomous driving that do more than avoid collisions and reach a goal: planners that resolve a priority-ordered rulebook---safety before regulatory compliance, comfort, and efficiency---at the vehicle's control rate. Such conflicts are routine: a vehicle may briefly cross a lane boundary to pass a stopped obstacle, and braking early for comfort can disrupt following traffic. Planners must also be auditable: after an objective conflict, an engineer should be able to identify which rule was relaxed and by how much~\citep{censi_rulebooks_2019,collin_sotif_rulebooks_2021}. These requirements are difficult to combine. Lexicographic schemes make priorities explicit, and cascaded strict-priority quadratic programs reach kilohertz rates in whole-body control~\citep{kanoun_taskpriority_2011,escande_hqp_2014}; however, when applied to nonconvex driving NMPC, a sequential formulation requires multiple nonlinear-program solves and tier-freezing at every cycle, increasing latency and disrupting warm starts. Learned planners scale well but expose limited per-rule structure for audit~\citep{dauner_2023_pdm,plantf_2023,pluto_2024}, whereas conventional penalty stacks do not directly reveal how conflicts were resolved.

\paragraph{W-SQP.} Our primary contribution is \textbf{W-SQP}, a weighted soft-constrained quadratic-penalty NMPC, also referred to as tiered-slack NMPC. W-SQP compiles nine rule families into a four-tier shared-slack nonlinear program built in CasADi and solved with IPOPT~\citep{casadi_2019,ipopt_2006}. Strongly separated penalties bias residual violations toward comfort and efficiency and away from safety and regulatory rules, while physical actuation bounds remain hard. This scalar penalty provides a priority bias, not lexicographic priority preservation (Sec.~\ref{sec:controller}). W-SQP replans from its executed state at $10$\,Hz; under a $90$\,ms solver-time limit, IPOPT returns its best iterate, which is projected through the dynamics before execution. Each cycle also records normalized per-rule residuals for audit.

\paragraph{Evaluating a non-imitative controller.} Because W-SQP plans from its executed state, it may select a valid path that differs from the recorded human trajectory. Closed-loop benchmarks evaluate executed states against reactive agents~\citep{nuplan_2021,gulino_waymax_2023}, but some catalogues also include path-tracking or route-adherence rules defined relative to the human log. Such rules measure imitation as well as driving quality, and aggregating them with physical and regulatory rules conflates ``drives like the recorded human'' with ``drives well''---a form of proxy failure related to Goodhart's law~\citep{manheim_goodhart_2018}. As a supporting contribution, Section~\ref{sec:fair} classifies metrics by their referenced state and decomposes the aggregate deficit into log-independent compliance and an imitation premium.

Taken together, two gaps emerge: a planning formulation that makes the resolution of competing priorities explicit and auditable while remaining computationally viable at the control rate, and an evaluation methodology that measures driving quality without rewarding resemblance to a single recorded trajectory. In this work, we address both.

\paragraph{Closed-loop study.} We evaluate W-SQP on $150$ WOMD scenarios in Waymax~\citep{gulino_waymax_2023,ettinger_womd_iccv_2021} against reactive and proposal-and-select baselines. Across the sixteen log-independent safety and regulatory rules, W-SQP has a mean difference of $-0.04\pp$ relative to expert replay, with no systematic group-level deficit. The two log-coupled rules account for a much larger mean difference of $-10.1\pp$; Cohen's $d=-0.62$ describes the contrast between these groups. This aggregate result does not imply per-rule parity: non-traversable-surface, speed-limit, and drivable-surface compliance decrease by $2.9$, $2.3$, and $2.1\pp$, respectively, primarily in high-divergence scenarios (Table~\ref{tab:two_deficits}).

The contributions are:
\begin{itemize}
  \item \textbf{W-SQP control architecture.} A rulebook-aware NMPC that maps nine rule families to four shared-slack tiers, retains hard actuation bounds, and logs normalized per-rule residuals every cycle (Sec.~\ref{sec:controller}).
  \item \textbf{Real-time and anytime characterization.} An analysis of NLP size, sparsity, and solver-time limits: $28$\,ms median and $104$\,ms maximum observed wall-clock time, with dynamics projection of the anytime iterate (Sec.~\ref{sec:realtime}).
  \item \textbf{Log-independent evaluation protocol.} A referenced-state taxonomy and imitation-premium decomposition separating driving quality from resemblance to a logged trajectory (Sec.~\ref{sec:fair}).
  \item \textbf{Closed-loop WOMD evaluation.} A $150$-scenario comparison with reactive and proposal-and-select baselines, with robustness controls and separate reporting of localized safety regressions (Sections~\ref{sec:setup} and~\ref{sec:results}).
\end{itemize}

The remainder of the paper is organized as follows. Section~\ref{sec:related} situates W-SQP within rule-aware planning, closed-loop benchmarking, and metric construction. Section~\ref{sec:controller} presents the controller, Section~\ref{sec:realtime} its real-time characterization, and Section~\ref{sec:fair} the log-independent evaluation protocol. Sections~\ref{sec:setup} and~\ref{sec:results} present the closed-loop study, Section~\ref{sec:discussion} discusses practical implications and limitations, and Section~\ref{sec:conclusion} concludes.

\section{Related Work}
\label{sec:related}
Our work sits at the intersection of rule-aware planning and control, closed-loop benchmarking, and compliance-metric construction. We first review the benchmarking literature that motivates the fair-evaluation protocol, then situate W-SQP within rule-aware planning.

\subsection{Closed-loop benchmarking and its metrics}
Closed-loop evaluation is now standard because open-loop log matching cannot capture feedback or interaction~\citep{nuplan_2021,dauner_2023_pdm}. nuPlan provides a large-scale planner benchmark with simulation interfaces~\citep{nuplan_2021}, and Waymax supports efficient JAX-based closed-loop rollouts on WOMD scenarios~\citep{gulino_waymax_2023,ettinger_womd_iccv_2021}. These ecosystems standardize planner rollout, but a bespoke scoring catalog may still combine collision, progress, comfort, and trajectory-matching terms without distinguishing driving quality from resemblance to the logged human. PDM~\citep{dauner_2023_pdm} demonstrates the limitations of open-loop scores; our complementary concern is that log-alignment terms reintroduce an imitation bias even within closed-loop evaluation.

\subsection{Agent realism vs.\ metric construction}
The closest recent work studies how the \emph{reactive agents} in the simulator affect benchmark outcomes. \citet{hagedorn_nuplan_reactive_2025} replace IDM with a learned reactive policy across many nuPlan planners and find that IDM-based simulation overestimates planning performance, with planner-dependent score declines and ranking shifts; \citet{peng_nuplan_r_2025} introduce nuPlan-R with diffusion-based reactive agents and observe that learning-based planners look comparatively better under realism. Both vary the \emph{agent-realism} axis and report aggregate-level effects. We vary a different axis---\emph{metric construction}---and isolate the effect per rule: a specific, identifiable subset of rules (log-alignment proxies) is structurally biased against non-imitative planners, the bias is present even under plain IDM agents, and it can be removed without changing the simulator. Our finding thus partially \emph{explains} the score shifts they report rather than competing with them: as agents become more reactive the human log becomes less followable, and any log-coupled sub-metric penalizes that divergence regardless of driving quality.

\subsection{Goodhart, the imitation gap, and training--evaluation mismatch}
Treating a proxy as a target is the classic setting of Goodhart's law~\citep{manheim_goodhart_2018}: once an imitation surrogate is optimized or scored as if it were the objective, it stops measuring the objective. Log-alignment rules are exactly such surrogates---an expert-replay baseline maximizes them by definition---so aggregating them with physical-compliance rules conflates ``drives like the logged human'' with ``drives well.''

This issue is related to, but distinct from, the \emph{imitation gap}, which can denote an expert--learner discrepancy caused by different observability~\citep{weihs_imitation_gap_2021} or, in driving, a discrepancy between planned motion and its realization by a tracker and vehicle dynamics~\citep{plantf_2023}; neither definition concerns scoring a valid alternative trajectory against a single logged demonstration. IGDrivSim~\citep{grislain_igdrivsim_2025} reports log divergence separately from overlap and off-road metrics but studies behavior-cloning policies rather than isolating reference dependence per rule. CLOVER~\citep{ang_clover_2026} identifies a related training--evaluation mismatch and uses it to improve trajectory generation; we instead analyze the metric construction itself. Causal confusion in imitation learning~\citep{dehaan_causal_confusion_2019} is also distinct: it concerns what a policy learns, whereas our analysis concerns how a policy is evaluated. Our contribution is the referenced-state taxonomy, per-rule empirical decomposition, and scoring recommendation, not the general observation that imitation and rule compliance can diverge.

\subsection{Rule-aware planning and the design of W-SQP}
Rulebooks~\citep{censi_rulebooks_2019} assign each rule a violation scalar and a priority order, with extensions to assurance~\citep{collin_sotif_rulebooks_2021} and optimal control~\citep{xiao_rulebased_oc_2021}. Priority intent is encoded in MPC via lexicographic and hierarchical schemes, $\varepsilon$-constraint methods, or penalty separation~\citep{lai_lexicographic_review_2023,abernethy_2024_lexopt,tercan_2024_tlo,mavrotas_2009,schwenzer_mpc_review_2021}; formal methods provide temporal-logic semantics, monitors, and repair~\citep{maierhofer_interstate_2020,maierhofer_intersection_2022,yamaguchi_rtamt_2024,lin_smt_repair_2024,halder_lexicographic_stl_2023}; control-barrier-function filters~\citep{ames_cbf_tac_2017,ames_cbf_ecc_2019,liu_iterative_dhocbf_2023,allamaa_resafe_2024} and Responsibility-Sensitive Safety (RSS)~\citep{shalevshwartz_rss_2017,hasuo_rss_2022} enforce safety feasibility; and learned planners~\citep{dauner_2023_pdm,plantf_2023} scale well but expose little per-rule structure. Closer to W-SQP, real-time-constrained MPC has been developed specifically for on-road autonomous driving, including collocation-based schemes that embed control-barrier safety constraints while keeping the solve within a closed-loop budget~\citep{allamaa_resafe_2024,liu_iterative_dhocbf_2023}. W-SQP shares its real-time, driving-specific emphasis but differs in its target: rather than certifying a single safety filter, it compiles a \emph{priority-ordered, multi-family} rulebook into a single tiered-slack program and exposes a per-rule audit of the resulting trade-offs.
Against this background, W-SQP uses established weighted-slack NMPC primitives rather than claiming a new planning-theory primitive or a lexicographic priority; its contribution is the integration of these primitives under real-time and auditability constraints. Section~\ref{sec:fair} then evaluates the resulting controller without treating log alignment as a safety or regulatory requirement.

\section{The W-SQP Controller}
\label{sec:controller}
We present \textbf{W-SQP}\footnote{W-SQP abbreviates \emph{weighted soft-constrained quadratic-penalty} NMPC, also called \emph{tiered-slack} NMPC. The nonlinear program is solved with IPOPT~\citep{ipopt_2006}; W-SQP does not use a sequential-quadratic-programming solver.}, a rulebook-aware receding-horizon controller that maps a heterogeneous driving rulebook to one online nonlinear program. It plans from the current executed state, biases soft-constraint relaxation by rule priority, and logs normalized per-rule residuals at each cycle; the constraints are smooth except for one clamped proximity gate (Sec.~\ref{subsec:surrogates}). The four tiers are optimized jointly rather than solved and frozen sequentially, so a sufficiently large lower-priority benefit can be traded off against a higher-priority violation: the formulation provides a penalty bias, not lexicographic priority preservation.

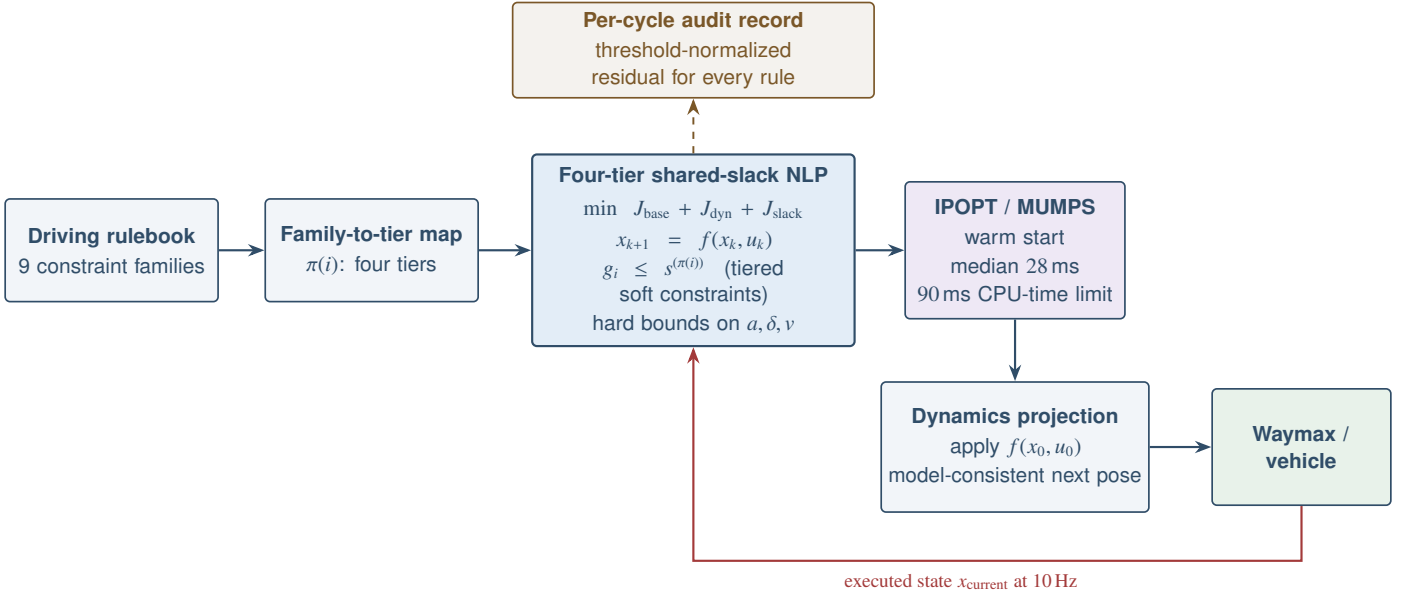
\begin{figure*}[t]
\centering
\begingroup
\definecolor{cBlock}{HTML}{F2F5F9}
\definecolor{cBorder}{HTML}{304D68}
\definecolor{cNLP}{HTML}{E3EDF7}
\definecolor{cSolver}{HTML}{EEE8F5}
\definecolor{cEnv}{HTML}{E8F2EA}
\definecolor{cAudit}{HTML}{7A5728}
\definecolor{cFlow}{HTML}{304D68}
\definecolor{cFB}{HTML}{A33A3A}
\resizebox{\textwidth}{!}{%
\begin{tikzpicture}[
  font=\sffamily\small,
  >={Stealth[length=2.4mm,width=1.8mm]},
  block/.style={
    draw=cBorder, line width=0.7pt, fill=cBlock, rounded corners=2pt,
    align=center, text=cBorder, inner sep=3pt, minimum height=15mm,
    text width=29mm},
  nlp/.style={
    block, fill=cNLP, text width=45mm, minimum height=28mm,
    line width=0.9pt},
  solver/.style={block, fill=cSolver, text width=30mm, minimum height=20mm},
  executor/.style={block, text width=37mm, minimum height=19mm},
  environment/.style={block, fill=cEnv, text width=24mm, minimum height=17mm},
  audit/.style={
    draw=cAudit, fill=cAudit!8, rounded corners=2pt, line width=0.7pt,
    align=center, text=cAudit, text width=50mm, inner sep=4pt},
  flow/.style={->, draw=cFlow, line width=0.9pt},
  feedback/.style={->, draw=cFB, line width=0.9pt},
  auditflow/.style={->, draw=cAudit, line width=0.8pt, dashed},
  label/.style={font=\footnotesize, align=center, fill=white, inner sep=2pt},
]

\node[block] (rulebook) at (0,0)
  {\textbf{Driving rulebook}\\[1.5pt]9 constraint families};
\node[block] (map) at (3.8,0)
  {\textbf{Family-to-tier map}\\[1.5pt]$\pi(i)$: four tiers};
\node[nlp] (nlp) at (8.5,0)
  {\textbf{Four-tier shared-slack NLP}\\[2.5pt]
   $\min\;J_{\mathrm{base}}+J_{\mathrm{dyn}}+J_{\mathrm{slack}}$\\[2pt]
   $x_{k+1}=f(x_k,u_k)$\\[1pt]
   $g_i\le s^{(\pi(i))}$ \; (tiered soft constraints)\\[1pt]
   hard bounds on $a,\delta,v$};
\node[solver] (solver) at (13.2,0)
  {\textbf{IPOPT / MUMPS}\\[1.5pt]
   warm start\\[1pt]
   median $28$\,ms\\[1pt]
   $90$\,ms CPU-time limit};

\draw[flow] (rulebook) -- (map);
\draw[flow] (map) -- (nlp);
\draw[flow] (nlp) -- (solver);

\node[executor, below=9mm of solver] (executor)
  {\textbf{Dynamics projection}\\[1.5pt]
   apply $f(x_0,u_0)$\\[1pt]
   model-consistent next pose};
\node[environment, right=9mm of executor] (environment)
  {\textbf{Waymax /}\\\textbf{vehicle}};

\draw[flow] (solver) -- (executor);
\draw[flow] (executor) -- (environment);

\node[audit, above=8mm of nlp] (audit)
  {\textbf{Per-cycle audit record}\\[1.5pt]
   threshold-normalized residual for every rule};
\draw[auditflow] (nlp.north) -- (audit.south);

\coordinate (fbturn) at ($(environment.south)+(0,-8mm)$);
\draw[feedback]
  (environment.south) -- (fbturn) -| (nlp.south)
  node[pos=0.28, label, text=cFB, below=1mm]
    {executed state $x_{\mathrm{current}}$ at $10$\,Hz};
\end{tikzpicture}%
}
\endgroup
\caption{W-SQP closed-loop architecture. Each cycle, the rulebook and current scene are compiled into differentiable surrogates assigned to four penalty tiers (safety $\succ$ regulatory $\succ$ comfort $\succ$ efficiency). IPOPT/MUMPS solves the multiple-shooting kinematic-bicycle NLP under a $90$\,ms CPU-time limit; the first control is propagated through the vehicle model before execution, normalized per-rule residuals are logged for audit, and the controller re-initializes from its executed state $x_{\mathrm{current}}$.}
\label{fig:cep_architecture}
\end{figure*}

\subsection{Model and decision variables}
The ego vehicle follows a kinematic-bicycle model with state $x_k=[p^x_k,p^y_k,\psi_k,v_k]^\top\in\R^4$ (position, yaw, and speed in the ego-local frame) and input $u_k=[a_k,\delta_k]^\top$ (longitudinal acceleration and steering). Fourth-order Runge--Kutta integration discretizes the continuous dynamics $\dot{x}=f_c(x,u)$ at $\Delta t=0.1$\,s, giving $x_{k+1}=f(x_k,u_k)$ with CasADi-differentiable expressions~\citep{casadi_2019}. At each replanning step, the NMPC optimizes over $N=30$ steps ($3.0$\,s) with decision variables $\mathbf{X}=\{x_0,\dots,x_N\}$, $\mathbf{U}=\{u_0,\dots,u_{N-1}\}$, and $\mathbf{S}$. The initial state $x_0$ is fixed to the current executed ego state rather than a logged pose.

\subsection{Objective}
The cost has three parts, $J=J_{\mathrm{base}}+J_{\mathrm{dyn}}+J_{\mathrm{slack}}$. The base term tracks a low-weight reference and penalizes effort and input rate,
\begin{equation}
J_{\mathrm{base}}=\sum_{k=0}^{N-1}\!\Bigl(\|x_k-x_k^{\mathrm{ref}}\|_Q^2+\|u_k\|_R^2+\|\Delta u_k\|_S^2\Bigr)+\alpha_T\|x_N-x_N^{\mathrm{ref}}\|_Q^2,
\label{eq:Jbase}
\end{equation}
with $\Delta u_k=u_k-u_{k-1}$; the $k{=}0$ terms (the input-rate and jerk penalties, and the jerk constraint below) use $u_{-1}$, the first optimized control of the previous cycle's solve carried in as a parameter under warm start (reset to zero on a solver fallback). The dynamics-shaping term discourages aggressive transients,
\begin{equation}
J_{\mathrm{dyn}}=\sum_{k=0}^{N-1}\!\Bigl(W_{\mathrm{j}}\bigl(\tfrac{a_k-a_{k-1}}{\Delta t}\bigr)^2+W_{\mathrm{a}}\,a_{\mathrm{lat},k}^2\Bigr)+\!\sum_{k=1}^{N-1}\!W_{\dot{\mathrm{a}}}\bigl(\tfrac{a_{\mathrm{lat},k}-a_{\mathrm{lat},k-1}}{\Delta t}\bigr)^2,
\label{eq:Jdyn}
\end{equation}
with $a_{\mathrm{lat},k}=v_k(\psi_{k+1}-\psi_k)/\Delta t$. The reference $x^{\mathrm{ref}}$ is the expert log, used \emph{only} as a low-weight stabilizing signal (Table~\ref{tab:solver_params}); it does not pin the ego to the log, and route adherence enters separately as one soft constraint below. This low weight is what makes W-SQP non-imitative: a representative $2$\,m log deviation costs about as much as a comfort-tier slack, so log alignment sits below the safety and regulatory tiers. ``Non-imitative'' means exactly this subordination, not the absence of any log reference. The slack term encodes priority through squared, weight-separated penalties,
\begin{equation}
J_{\mathrm{slack}}=\sum_{j=1}^{4}\rho_j\sum_{k=0}^{N-1}\bigl(s_k^{(j)}\bigr)^2,\quad \rho_1{\gg}\rho_2{\gg}\rho_3{\gg}\rho_4,
\label{eq:Jslack}
\end{equation}
ordering the four tiers safety $\succ$ regulatory $\succ$ comfort $\succ$ efficiency. This scalarization is a feasibility-preserving relaxation that keeps residual violations low and biased toward low-priority families.

We penalize the slacks quadratically ($\ell_2$) rather than with the \emph{exact} $\ell_1$ penalty of the soft-constrained-MPC literature~\citep{kerrigan_soft_2000,scokaert_feasibility_1999}. The trade-off is deliberate. The $\ell_1$ form has a theoretical advantage: above a finite weight threshold, it satisfies each tier's constraints \emph{exactly} whenever they are feasible, and yields the least violation otherwise. The $\ell_2$ form gives this up---because $\nabla(\rho s^2)\to0$ as $s\to0$, small violations become asymptotically free, and even Tier-1 surrogates settle just below $100\%$ (Table~\ref{tab:per_rule_compliance}). What $\ell_2$ buys in return is numerical: it is strictly convex and $C^2$ in the slacks, it removes the dual degeneracy that $\ell_1$ shows at its exact-penalty threshold, and it produces a better-conditioned barrier Hessian. Across the $28{,}131$ warm-started solves at $10$\,Hz this made IPOPT markedly more robust, which is why we chose it. An $\ell_1$ variant---which would additionally let us check the lexicographic-equivalence conditions per solve from IPOPT's returned multipliers---is a natural refinement for future work.

\begin{figure}[t]
\centering
\begingroup
\definecolor{tierone}{RGB}{184,38,38}
\definecolor{tiertwo}{RGB}{211,105,24}
\definecolor{tierthree}{RGB}{32,112,164}
\definecolor{tierfour}{RGB}{31,112,72}
\resizebox{\columnwidth}{!}{%
\begin{tikzpicture}[
  font=\sffamily\small,
  >={Stealth[length=2.0mm,width=1.5mm]},
  family/.style={
    draw=black!55, fill=black!4, rounded corners=1.5pt, line width=0.55pt,
    minimum width=29mm, minimum height=5.6mm, inner sep=1pt,
    align=center, font=\footnotesize},
  tier/.style={
    rounded corners=2pt, line width=0.8pt, minimum width=38mm,
    minimum height=11.5mm, inner sep=2.5pt, align=center},
  mapping/.style={->, draw=black!52, line width=0.6pt},
  weight/.style={
    rounded corners=1.5pt, line width=0.7pt, minimum width=18mm,
    minimum height=10mm, inner sep=2pt, align=center, font=\footnotesize},
  note/.style={
    draw=black!30, fill=black!3, rounded corners=1.5pt,
    text width=78mm, align=left, inner sep=3.5pt, font=\footnotesize},
]

\node[font=\small\bfseries, align=center] at (0,0.9)
  {Constraint families\\priority high $\downarrow$ low};
\node[font=\small\bfseries, align=center] at (4.9,0.9)
  {Shared-slack tiers};

\node[family] (collision) at (0,0) {collision};
\node[family, below=1.2mm of collision] (edge) {road edge};
\node[family, below=1.2mm of edge] (speed) {speed limit};
\node[family, below=1.2mm of speed] (light) {traffic light};
\node[family, below=1.2mm of light] (route) {route};
\node[family, below=1.2mm of route] (decel) {deceleration};
\node[family, below=1.2mm of decel] (jerk) {jerk};
\node[family, below=1.2mm of jerk] (latacc) {lateral acceleration};
\node[family, below=1.2mm of latacc] (headway) {headway};

\node[tier, draw=tierone, fill=tierone!8, text=tierone!75!black]
  (t1) at (4.9,-0.39)
  {\textbf{Tier 1: Safety}\\[1pt]\footnotesize collision, road edge};
\node[tier, draw=tiertwo, fill=tiertwo!9, text=tiertwo!75!black]
  (t2) at (4.9,-2.10)
  {\textbf{Tier 2: Regulatory}\\[1pt]\footnotesize speed, light, route};
\node[tier, draw=tierthree, fill=tierthree!8, text=tierthree!75!black]
  (t3) at (4.9,-4.20)
  {\textbf{Tier 3: Comfort}\\[1pt]\footnotesize decel., jerk, lat. accel.};
\node[tier, draw=tierfour, fill=tierfour!9, text=tierfour!75!black]
  (t4) at (4.9,-5.60)
  {\textbf{Tier 4: Efficiency}\\[1pt]\footnotesize headway};

\draw[mapping] (collision.east) -- (t1.west);
\draw[mapping] (edge.east) -- (t1.west);
\draw[mapping] (speed.east) -- (t2.west);
\draw[mapping] (light.east) -- (t2.west);
\draw[mapping] (route.east) -- (t2.west);
\draw[mapping] (decel.east) -- (t3.west);
\draw[mapping] (jerk.east) -- (t3.west);
\draw[mapping] (latacc.east) -- (t3.west);
\draw[mapping] (headway.east) -- (t4.west);

\node[font=\small\bfseries, anchor=west] at (-1.45,-6.50)
  {Quadratic penalty weights};
\node[weight, draw=tierone, fill=tierone!10, text=tierone!75!black]
  (w1) at (-0.95,-7.5) {\textbf{T1}\\$\rho_1=10^8$};
\node[weight, draw=tiertwo, fill=tiertwo!11, text=tiertwo!75!black,
      right=1.5mm of w1]
  (w2) {\textbf{T2}\\$\rho_2=10^5$};
\node[weight, draw=tierthree, fill=tierthree!10, text=tierthree!75!black,
      right=1.5mm of w2]
  (w3) {\textbf{T3}\\$\rho_3=10^2$};
\node[weight, draw=tierfour, fill=tierfour!10, text=tierfour!75!black,
      right=1.5mm of w3]
  (w4) {\textbf{T4}\\$\rho_4=1$};
\draw[->, line width=0.8pt, black!55]
  ($(w1.south west)+(0,-2.5mm)$) --
  node[below=0.8mm, font=\footnotesize, fill=white, inner sep=1pt]
    {decreasing priority bias}
  ($(w4.south east)+(0,-2.5mm)$);

\node[note, below=8.5mm of w2, xshift=10mm]
  {\textbf{One shared slack per tier per step:} $4N=120$ slacks for $N=30$.\\[1.5pt]
   \textbf{Hard bounds:} $(a,\delta,v)$ are never relaxed. The weights
   \textbf{bias} violations toward lower tiers; they do not impose a
   lexicographic ordering.};
\end{tikzpicture}%
}
\endgroup
\caption{From priority-ordered rulebook to tiered shared slacks. The nine rule families collapse onto four penalty tiers; each tier shares one nonnegative slack per step, penalized by a weight $\rho_j$ separated by orders of magnitude. The ladder \emph{biases}---but does not lexicographically force---residual violations toward the low-priority tiers.}
\label{fig:cep_tiered_slack}
\end{figure}
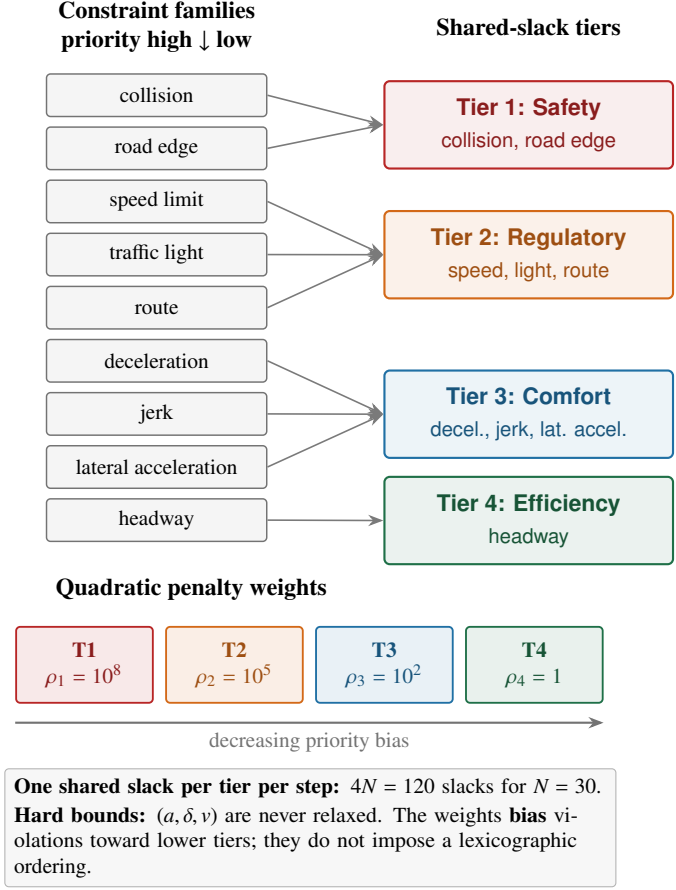

\subsection{Tiered soft constraints and the NLP}
Fig.~\ref{fig:cep_tiered_slack} shows the mapping from rulebook to program: each priority family maps to one of four tiers, and all constraints in a tier share a single nonnegative slack per step whose penalty weight $\rho_j$ is separated by orders of magnitude. Each rule surrogate is an inequality $g_i(x_k,u_k)\le0$, relaxed by a single shared slack per tier per step via a fixed family-to-tier map $\pi(i)\in\{1,2,3,4\}$:
\begin{equation}
g_i(x_k,u_k)\;\le\;s_k^{(\pi(i))},\qquad s_k^{(j)}\ge0 .
\label{eq:soft}
\end{equation}
W-SQP solves, at each step,
\begin{subequations}
\label{eq:nlp}
\begin{align}
\min_{\mathbf{X},\mathbf{U},\mathbf{S}}\;& J_{\mathrm{base}}+J_{\mathrm{dyn}}+J_{\mathrm{slack}}\\
\text{s.t.}\;& x_{k+1}=f(x_k,u_k),\quad x_0=x_{\mathrm{current}},\\
& g_i(x_k,u_k)\le s_k^{(\pi(i))},\;\; s_k^{(j)}\ge0,\quad\forall i,j,k,\\
& a_{\min}\le a_k\le a_{\max},\;|\delta_k|\le\delta_{\max},\;0\le v_k\le v_{\max},
\end{align}
\end{subequations}
where the last line contains hard physical bounds on acceleration, steering, and speed. Collision avoidance and the road boundary are, by contrast, Tier-1 \emph{soft} constraints: they carry the largest penalty $\rho_1$ but may in principle be relaxed---the sub-$100\%$ safety-rule compliance in Table~\ref{tab:per_rule_compliance} reflects exactly this. The single-slack-per-tier design yields only $4N=120$ slacks (vs.\ one per constraint), which is what makes $10$\,Hz replanning tractable; it couples within-tier constraints, so the slack tracks the most binding active constraint in its tier. We solve \eqref{eq:nlp} with IPOPT~\citep{ipopt_2006} (MUMPS), warm-started by the time-shifted previous solution.

\subsection{Constraint surrogates}
\label{subsec:surrogates}
Let $p_k=[p^x_k,p^y_k]^\top$ and $p_k^{(m)}$ the position of tracked object $m\in\{1,\dots,M\}$. The families and their evaluator-rule coverage are in Table~\ref{tab:constraint_families}.

\noindent\emph{Collision (Tier 1)} and \emph{road boundary (Tier 1):}
\begin{align}
g_{\mathrm{col}}&=d_{\mathrm{safe}}^2-\|p_k-p_k^{(m)}\|_2^2\le s_k^{(1)},\label{eq:collision}\\
g_{\mathrm{edge}}&=n_k^\top(p_k-b_k)-(w_{\mathrm{ego}}+m_{\mathrm{edge}})\le s_k^{(1)},
\end{align}
with $b_k$ the nearest road-edge point and $n_k$ its outward normal.

\noindent\emph{Speed limit, traffic light, route (Tier 2):}
\begin{align}
g_{\mathrm{spd}}&=v_k-v_{\max,k}\le s_k^{(2)},\\
g_{\mathrm{tl}}&=w_R(d_k^{(j)})\,(v_k-v_{\mathrm{stop}})\le s_k^{(2)},\\
g_{\mathrm{route}}&=\|p_k-p_k^{\mathrm{ref}}\|_2-d_{\mathrm{route}}\le s_k^{(2)} .\label{eq:route}
\end{align}
The traffic-light gate uses a clamped proximity weight $w_R(d)=\max\!\bigl(0,\,1-d/R\bigr)$ for stop distance $d_k^{(j)}$ and approach radius $R$; clamping at $0$ makes the gate inactive beyond $R$. The clamp is necessary: without it, $d>R$ gives $w_R<0$, so the term $w_R(v_k-v_{\mathrm{stop}})$ \emph{decreases} with speed and would reward acceleration toward the light. This $\max(0,\cdot)$ is the one non-smooth surrogate; we pass its $C^0$ kink to IPOPT as-is, since it sits off the active set whenever the gate is inactive. None of the solver failures behind the $0.58\%$ open-loop-fallback rate (Sec.~\ref{sec:realtime}) localizes at the kink, and traffic-light compliance ($^{7}r_{1}$) is at parity ($+0.3\pp$).

\noindent\emph{Comfort, jerk, lateral acceleration (Tier 3):}
\begin{align}
g_{\mathrm{dec}}&=-a_k-a_{\mathrm{comf}}\le s_k^{(3)},\\
g_{\mathrm{jerk}}^{\pm}&=\pm\tfrac{a_k-a_{k-1}}{\Delta t}-j_{\max}\le s_k^{(3)},\\
g_{\mathrm{lat}}&=\bigl(\tfrac{v_k^2}{L}\tan\delta_k\bigr)^2-a_{\mathrm{lat,max}}^2\le s_k^{(3)} .
\end{align}

\noindent\emph{Headway (Tier 4):} with a per-agent validity flag $\nu_m\in\{0,1\}$ (unity for a tracked agent, zero otherwise) and the center-to-center distance to agent $m$'s constant-velocity prediction $p_k^{(m)}$,
\begin{equation}
g_{\mathrm{hw}}=\nu_m\bigl(T_{\mathrm{hw}}v_k-\lVert p_k-p_k^{(m)}\rVert\bigr)\le s_k^{(4)} .
\end{equation}
The flag $\nu_m$ only deactivates absent agents. The surrogate is deliberately simple: it uses the center-to-center distance rather than an along-track projection, so at $v{=}15$\,m/s it activates on any tracked agent within $T_{\mathrm{hw}}v{=}22.5$\,m in \emph{any} direction---a weak repulsive potential that the lowest (Tier-4) weight keeps benign. We treat this crude proxy, like the constant-velocity prediction it depends on, as an experimental variable rather than tuning it away (Sec.~\ref{sec:setup}).

\subsection{Surrogate vs.\ evaluator}
\label{subsec:surrogate_vs_evaluator}
W-SQP enforces the differentiable surrogates above, not the evaluator's $25$ rules directly, and the geometries differ (center-to-center vs.\ corner-to-corner collision distance; a norm-based route surrogate vs.\ the evaluator's position-and-heading thresholds). The evaluator remains authoritative. We examine the surrogate--evaluator mismatch directly in Sec.~\ref{sec:results} so that the log-coupling result cannot be an artifact of the controller's internal geometry.

\paragraph{Rules without a dedicated surrogate.}
Twelve of the evaluator's $25$ rules---including stop-sign compliance ($^{8}r_{0}$), crosswalk yield ($^{8}r_{1}$), yield-to-priority ($^{1}r_{0}$), lane intrusion ($^{3}r_{6}$), and the two advisory-only rules---have no dedicated term in the NLP~\eqref{eq:nlp}. Where W-SQP scores well on them anyway (Table~\ref{tab:per_rule_compliance}), the reason is one of two. Some are rarely \emph{applicable} in the scenario set---a stop sign or an occupied crosswalk within range---so their compliance is computed over very few steps and says little about capability. Others are satisfied \emph{incidentally}: yield-to-priority and lane intrusion, for instance, fall out of the Tier-1/Tier-2 collision, road-boundary, and route surrogates. Either way this is an empirical observation about the evaluated scenarios, not a design property, and we list it among the limitations of Sec.~\ref{subsec:probe_artifact}.

\begin{table}[t]
\centering
\caption{Constraint families in the W-SQP NLP and the evaluator rules they cover. Horizon $N=30$, $M=3$ tracked agents, $K_{\mathrm{tl}}=2$ traffic-light stop points.}
\label{tab:constraint_families}
\small
\renewcommand{\arraystretch}{1.1}
\begin{tabular}{p{3.1cm} c r p{2.0cm}}
\toprule
Constraint family       & Tier & Count                                    & Evaluator rules \\
\midrule
Collision avoidance     & 1    & $N{\cdot}M{=}90$                         & $^{10}r_{0}$, $^{9}r_{0}$ \\
Road boundary           & 1    & $N{+}1{=}31$                             & $^{9}r_{1}$, $^{7}r_{0}$ \\
Speed limit             & 2    & $N{+}1{=}31$                             & $^{3}r_{0}$ \\
Traffic-light stop      & 2    & $N{\cdot}K_{\mathrm{tl}}{=}60$           & $^{7}r_{1}$ \\
Route adherence         & 2    & $N{+}1{=}31$                             & $^{2}r_{2}$ \\
Comfort deceleration    & 3    & $N{=}30$                                 & $^{0}r_{2}$ \\
Jerk ($\pm$)            & 3    & $2N{=}60$                                & $^{1}r_{9}$ \\
Lat.\ accel             & 3    & $2N{=}60$                                & $^{1}r_{11}$ \\
Headway                 & 4    & $N{\cdot}M{=}90$                         & $^{3}r_{3}$ \\
\bottomrule
\end{tabular}
\end{table}

\subsection{Controller characterization and limitations}
\label{subsec:probe_artifact}
The per-cycle slack log makes the controller's relaxation decisions inspectable after the fact. On a single scenario the safety and regulatory tiers stay inactive throughout while relaxation is confined to comfort and efficiency, and this holds across the fleet as well (Fig.~\ref{fig:cep_audit_panel}, Sec.~\ref{sec:results})---the priority bias read off the log rather than asserted.

\emph{Acknowledged limitations.} A single slack per tier tracks only the most binding constraint in that tier, so we audit each surrogate's normalized residual rather than reading the active rule off the shared slack. Across the fleet, the route surrogate is relaxed on about $17\%$ of steps, against $2.5\%$ for collision, $4.9\%$ for road edge, and $1.6\%$ for speed (Fig.~\ref{fig:cep_audit_panel}). Constraints within a tier also carry different units---Tier~1 mixes a squared-distance collision residual (m$^2$) with a signed road-edge distance (m)---so raw slack magnitudes are neither comparable nor deliberately balanced; nondimensionalizing each surrogate by its threshold, $\tilde g_i=g_i/\tau_i$, would fix this. The remaining limitations are the lack of a lexicographic guarantee, constant-velocity agent prediction, surrogate--evaluator mismatch (including the twelve rules with no surrogate above), and a single-disk collision model that under-covers long vehicles.

\section{Real-Time Implementation and Computational Analysis}
\label{sec:realtime}
Real-time operation is a design requirement for W-SQP. This section reports the controller's computational cost, the empirical solve-time tail, behavior under a solver-time limit, and the dynamics projection used prior to execution.

\subsection{Compute target and problem size}
All timing figures are measured on a single-workstation target---an Intel Core i9-12950HX (16 cores, up to $5.0$\,GHz, $64$\,GB RAM)---running IPOPT single-threaded with the MUMPS linear solver, so the numbers reflect one solver instance rather than a parallel farm. Re-profiling on an embedded automotive target is future work (Sec.~\ref{sec:conclusion}); the aim here is to characterize the algorithm's intrinsic cost and tail behavior, not to certify a specific ECU.

The tractability of that cost rests on structure. At the fixed settings of Table~\ref{tab:solver_params} ($N{=}30$, up to $M$ tracked agents) the transcribed NLP has $304$ decision variables and $818$ constraints, and its constraint Jacobian $\partial g/\partial w$ carries only $2381$ nonzeros---a density of $0.96\%$ (Fig.~\ref{fig:cep_sparsity}). The banded, near-block-diagonal pattern is the multiple-shooting signature: each stage couples only to its neighbor through the dynamics defect, and the tiered soft constraints add sparse rows rather than dense coupling. IPOPT with MUMPS exploits exactly this sparsity, which is why a nominally $304$-variable nonconvex program can be solved in tens of milliseconds. The $10^{8}$--$10^{0}$ span of the tier penalties is left to IPOPT's default gradient-based NLP scaling with an adaptive barrier strategy; the heavy uncapped solve-time tail concentrates in the same hard, high-divergence scenes---those with the worst-conditioned barrier Hessians---in which the anytime cap binds.

\begin{figure}[t]
\centering
\includegraphics[width=0.72\columnwidth]{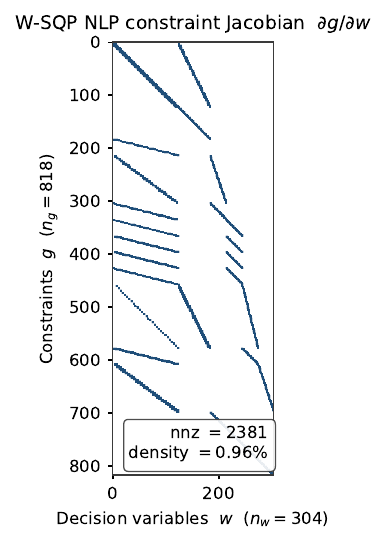}
\caption{Sparsity pattern of the W-SQP NLP constraint Jacobian $\partial g/\partial w$ at the settings of Table~\ref{tab:solver_params} ($304$ variables, $818$ constraints, $2381$ nonzeros, $0.96\%$ density). The banded multiple-shooting structure is what IPOPT/MUMPS exploits to solve the nonconvex program at the loop rate.}
\label{fig:cep_sparsity}
\end{figure}

\subsection{Typical solve time}
Across the $28{,}131$ closed-loop solves of our $n{=}150$ evaluation (Sec.~\ref{sec:results}), warm-started W-SQP solves in a \emph{median} of $28.4$\,ms---comfortably within a $100$\,ms budget---with a \emph{mean} of $77.6$\,ms inflated by a heavy tail (p95 $230$\,ms, p99 $1.08$\,s, max $5.74$\,s). The controller therefore typically meets $10$\,Hz, but not worst-case; a gentle open-loop braking fallback covers the $0.58\%$ of solves ($163/28{,}131$) that exceed tolerance or fail, so tail latency degrades gracefully rather than stalling the rollout. The fallback also cannot drive the log-coupled result of Sec.~\ref{sec:fair}: at $0.58\%$ of solve steps, even if \emph{every} fallback step maximally violated a log-coupled rule, it could shift that rule's compliance by at most ${\approx}0.6\pp$---negligible against the $-10.1\pp$ deficit.

\subsection{Observed latency under a solver-time limit}
We re-run all $150$ scenarios with IPOPT's \texttt{max\_cpu\_time} set to $90$\,ms. When the limit is reached, IPOPT returns its best iterate and the controller records a budget hit. This option is not a strict wall-clock deadline: IPOPT checks it at iteration boundaries, allowing the current iteration to finish, and it measures CPU rather than wall-clock time. On our single-threaded configuration, the two measures track closely. The limit reduces the observed tail to a $28.5$\,ms median, $39.6$\,ms mean, $96$\,ms p99, and $104$\,ms maximum, compared with a $5.74$\,s maximum without the limit; $11.9\%$ of solves reach the limit (Fig.~\ref{fig:cep_solvetime}).

\begin{figure}[t]
\centering
\includegraphics[width=\columnwidth]{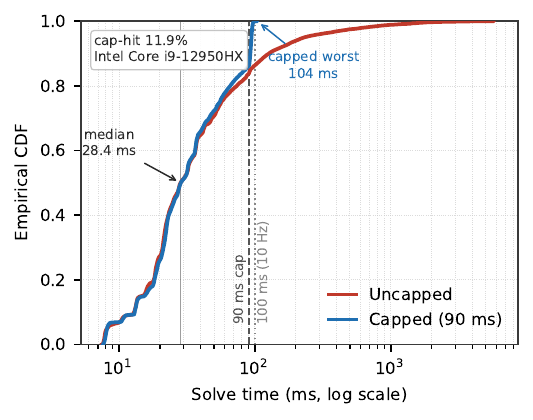}
\caption{Empirical CDF of per-solve wall-clock time over the $n{=}150$ evaluation (single-threaded IPOPT/MUMPS, log-scale abscissa). A $90$\,ms CPU-time limit reduces the maximum from $5.74$\,s to $104$\,ms while leaving the median near $28$\,ms; $11.9\%$ of solves return a best iterate. IPOPT checks the limit at iteration boundaries, so this is an empirical tail reduction rather than a strict wall-clock guarantee.}
\label{fig:cep_solvetime}
\end{figure}
The log-coupling decomposition is nearly unchanged under the solver-time limit ($\overline{\Delta c}_{\log}=-10.8\pp$, $\overline{\Delta c}_{\mathrm{inv}}=+0.4\pp$), so the fair-evaluation result of Sec.~\ref{sec:fair} is not an artifact of unbounded solve time. We therefore describe W-SQP as \emph{anytime-capable}: it usually meets the $10$\,Hz loop period and returns a best iterate when the solver-time limit is reached. A hard-real-time claim would require a platform-level deadline mechanism and worst-case execution-time analysis.

The best iterate retains the hard bounds. The interior-point iterate stays strictly within the hard box bounds on acceleration, steering, and speed (the barrier enforces them, to IPOPT's default $10^{-8}$ bound-relaxation tolerance), and the executed pose's \emph{lateral} acceleration and jerk are additionally hard-clamped in post-processing (to $2.8$\,m/s$^2$ and $2.5$\,m/s$^3$, via a bound on the executed heading increment); ``never relaxes hard physical bounds'' refers to these actuation and box limits, which hold on every step, cap-hit or not. Longitudinal jerk, by contrast, is governed only by the soft Tier-3 slack, which is why the evaluator's longitudinal-jerk rule ($^{1}r_{9}$) shows a small residual deficit ($-3.7\pp$, Table~\ref{tab:per_rule_compliance}) while the lateral comfort rules stay at parity.

What a non-converged iterate does \emph{not} guarantee is exact satisfaction of the dynamics equality $x_{k+1}=f(x_k,u_k)$. Measuring the executed-step defect $\lVert x_1-f(x_0,u_0)\rVert$ over every cap-hit solve in a dedicated $50$-scenario instrumented run ($9{,}374$ solves, $1{,}102$ cap-hits at $11.8\%$): $34\%$ of cap-hit iterates are already exactly feasible and the median defect is below $1$\,mm, but the tail is heavy---p95 $0.20$\,m, p99 $0.74$\,m, worst $15.5$\,m---because where the cap binds almost every step, the raw best iterate can be grossly model-inconsistent.

We therefore \emph{project} the anytime iterate onto the dynamics before execution: the executed pose is $f(x_0,u_0)$, the applied control rolled through the vehicle model, so every executed pose is dynamically feasible by construction and ``never relaxes hard physical bounds'' holds for both the box bounds and the dynamics (Fig.~\ref{fig:cep_projection}). Projection is a no-op for converged solves and alters only the ${\approx}11\%$ cap-hit steps. Under it the capped decomposition is essentially unchanged---$\overline{\Delta c}_{\log}=-9.9\pp$, $\overline{\Delta c}_{\mathrm{inv}}=+0.4\pp$, versus $-10.8/{+}0.4$ raw and $-10.1/{-}0.04$ uncapped---so the fair-evaluation result is not an artifact of the anytime executor; if anything, projection slightly shrinks the log-coupled deficit by removing the raw iterate's rare teleport excursions.

\begin{figure}[t]
\centering
\includegraphics[width=\columnwidth]{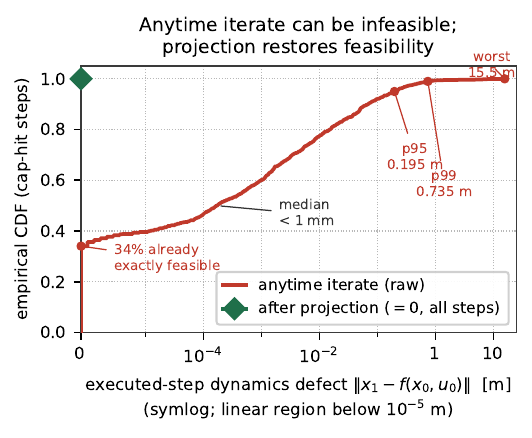}
\caption{Executed-step dynamics defect $\lVert x_1-f(x_0,u_0)\rVert$ over the $1{,}102$ cap-hit steps of the instrumented run (symlog abscissa; linear region below $10^{-5}$\,m so exact zeros are shown). The raw anytime iterate (red) is exactly feasible on $34\%$ of cap-hit steps and $<1$\,mm at the median but has a heavy tail; rolling the applied control through the vehicle model collapses the executed defect to exactly $0$ on every step (green point mass).}
\label{fig:cep_projection}
\end{figure}

\subsection{Sensitivity to the solver-time limit}
To assess sensitivity to available compute, we sweep the CPU-time limit over $\{30,50,70,90,120\}$\,ms and include an unlimited condition on a $40$-scenario subset (Fig.~\ref{fig:cep_budget}). The log-independent safety/regulatory difference remains near parity throughout the sweep: it is $-0.14\pp$ at $30$\,ms, where $55\%$ of solves reach the limit. The log-coupled difference changes from $-11.6\pp$ at $90$\,ms to $-18.5\pp$ at $30$\,ms. The budget-hit rate rises smoothly from $12\%$ at $120$\,ms to $55\%$ at $30$\,ms. These data indicate graceful degradation on this subset, but they do not replace profiling on an embedded automotive target.

\begin{figure}[t]
\centering
\includegraphics[width=\columnwidth]{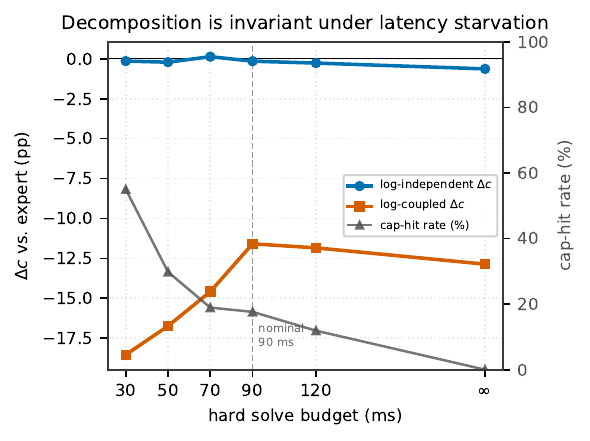}
\caption{Solver-time-limit sweep on a $40$-scenario subset. As the limit decreases from unlimited to $30$\,ms, the log-independent safety/regulatory difference remains near parity (blue, left axis), the log-coupled difference becomes more negative (orange), and the budget-hit rate increases smoothly (grey, right axis).}
\label{fig:cep_budget}
\end{figure}

\section{A Fair, Log-Independent Evaluation Protocol}
\label{sec:fair}
W-SQP may follow a valid trajectory that differs from expert replay. To separate driving quality from resemblance to that replay, we classify closed-loop metrics by the state they reference and decompose aggregate compliance into log-independent and log-coupled terms. Section~\ref{sec:results} uses this protocol to report both terms rather than merging them into one score.

\subsection{Closed-loop rollouts}
\label{subsec:rollouts}
A scenario $\omega\in\mathcal{D}$ fixes an initial state, a static environment $e_\omega$ (map and traffic-control state), and a \emph{recorded} (logged) ego trajectory $\xi^{\star}_\omega=(s^{\star}_0,\dots,s^{\star}_T)$ with planar positions $p^{\star}_k$. An ego policy is a causal map $\pi\colon\mathcal{S}_k\mapsto a_k$. Rolling out $\pi$ in closed loop with a reactive non-ego policy $\pi_{\mathrm{NPC}}$~\citep{waymax_agents_doc} at $10$\,Hz produces the \emph{executed} state sequence
\begin{equation}
\Gamma(\pi,\omega)\;=\;\bigl(\xi^{\pi}_\omega,\;\zeta^{\pi}_\omega,\;e_\omega\bigr),
\label{eq:rollout}
\end{equation}
where $\xi^{\pi}_\omega$ is the executed ego trajectory (positions $p^{\pi}_k$) and $\zeta^{\pi}_\omega$ the executed non-ego trajectories. Because the agents are reactive, $\zeta^{\pi}_\omega$ depends on $\pi$: a planner that deviates from $\xi^{\star}_\omega$ elicits different agent responses. Expert replay is the policy $\pi_{\star}$ that reproduces the log, $\xi^{\pi_{\star}}_\omega=\xi^{\star}_\omega$. All metrics below are functionals of $\Gamma(\pi,\omega)$ alone (evaluation on executed states only), preventing leakage~\citep{nuplan_2021,dauner_2023_pdm}.

\subsection{Compliance metrics as functionals}
Each rule $r$ defines a \emph{violation predicate} $\phi_r$ and an \emph{applicability indicator} $\alpha_r$, both evaluated per step, and a per-rule compliance
\begin{equation}
c_r(\pi,\omega)\;=\;100\cdot\Bigl(1-\tfrac{1}{T}\textstyle\sum_{k=0}^{T-1}\alpha_r(\mathcal{S}^{\pi}_k)\,\phi_r(\mathcal{S}^{\pi}_k)\Bigr)\in[0,100].
\label{eq:rule_compliance}
\end{equation}
The indicator $\alpha_r$ zeroes steps at which the rule does not apply, so the sum counts only \emph{applicable} violations---but the normalization is over all $T$ steps. Compliance $c_r$ therefore blends how often a rule applies with its conditional violation rate, and a rarely applicable rule scores high almost automatically. Under our paired, same-scenario evaluation both policies face nearly identical applicability, so $c_r$ is a sound \emph{relative} measure even though it is not an absolute safety rate---which is why some rules flag even expert replay at high rates (Sec.~\ref{subsec:limitations}). All our claims are paired differences, for which this normalization is appropriate.
The catalogue has $|\mathcal{R}|=25$ rules over $9$ populated priority levels; an aggregate is the weighted mean $C(\pi,\omega)=\sum_{r}w_r\,c_r$, as used in leaderboards. The evaluator's geometry is independent of any planner surrogate (corner-to-corner bounding-box distances, signed distance to road-edge polylines), with a fixed strictness multiplier $\kappa=1.25$ applied identically to both policies. Here ``independent'' means independent of the planner's differentiable surrogates; the evaluator is part of the same rulebook framework, not a third-party benchmark.

\subsection{A taxonomy by referenced state}
\label{sec:taxonomy}
The decisive question for fairness is \emph{what part of the world a predicate reads}. We split the rollout arguments of $\phi_r$ into the executed driving $D_\omega^{\pi}\triangleq(\xi^{\pi}_\omega,\zeta^{\pi}_\omega,e_\omega)$---everything physically realized---and the recorded log $\xi^{\star}_\omega$.

\begin{definition}[Log-coupled metric]
\label{def:logcoupled}
A metric $c_r$ is \emph{log-coupled} if its predicate references the recorded log: there exist two rollouts with identical executed driving $D^{\pi}_\omega$ but different logs $\xi^{\star}_\omega\neq\tilde\xi^{\star}_\omega$ that yield $c_r\neq\tilde c_r$. Equivalently $\phi_r=\phi_r(D^{\pi}_\omega;\xi^{\star}_\omega)$ with nontrivial dependence on $\xi^{\star}_\omega$. We write $r\in\mathcal{R}_{\log}$.
\end{definition}

A canonical log-coupled rule is the deviation predicate
\begin{equation}
\phi_r(\mathcal{S}^{\pi}_k)=\ind\!\left[d(p^{\pi}_k,p^{\star}_k)>\tau_r\right],
\end{equation}
where $\tau_r$ is a threshold and $d$ is a pose distance. Path tracking ($^{0}r_{0}$) and route adherence ($^{2}r_{2}$) have this form. In our evaluator, their position thresholds are $2.4$ and $4.0$\,m, respectively, and route adherence also uses a $0.4$\,rad heading threshold. These values include the strictness factor $\kappa=1.25$. Route adherence is log-coupled here because our evaluator defines it relative to the logged pose; route adherence defined against a lane graph or assigned corridor would instead be map-defined and log-independent.

\begin{definition}[Log-independent metric]
\label{def:logindep}
$c_r$ is \emph{log-independent} if $\phi_r=\phi_r(D^{\pi}_\omega)$ depends on the rollout only through the executed driving and never references $\xi^{\star}_\omega$. We write $r\in\mathcal{R}_{\mathrm{ind}}$.
\end{definition}

Equivalently, a log-independent metric assigns the same score to two rollouts that differ only in \emph{which} compliant path the planner chose; log-coupled metrics are exactly those that do not.

Log-independent rules subdivide by which part of $D^{\pi}_\omega$ they read: \emph{map-defined} (static map/traffic control, e.g.\ drivable surface, traffic-light, stop sign), \emph{interaction-defined} (other agents $\zeta^{\pi}_\omega$, e.g.\ collision, headway, clearance), and \emph{controller-defined} (ego smoothness only, e.g.\ jerk, comfort). For analysis we group the safety/regulatory log-independent rules into the set $\mathcal{R}_{\mathrm{inv}}$ ($16$ rules), report controller-defined rules ($\mathcal{R}_{\mathrm{ctrl}}$, $7$ rules) separately as an intermediate group, and isolate $\mathcal{R}_{\log}$ ($2$ rules). We refer to $\mathcal{R}_{\mathrm{inv}}$ as \emph{log-independent} as shorthand for \emph{log-independent safety/regulatory}: the group's compliance is invariant to whether the ego imitates the log, \emph{not} insensitive to interaction---collision, headway, and clearance plainly depend on the other agents. The partition is fixed \emph{a priori} from each rule's definition, independent of any result.

\subsection{The log-coupling confound}
The following two facts are elementary but underpin the fair comparison of Sec.~\ref{sec:results}: they show a log-coupled metric measures imitation, not driving quality.

\begin{proposition}[Expert-replay dominance]
\label{prop:dominance}
For every log-coupled deviation rule $r\in\mathcal{R}_{\log}$ and every scenario $\omega$,
\begin{equation}
c_r(\pi_{\star},\omega)=100\;\ge\;c_r(\pi,\omega)\qquad\text{for all policies }\pi,
\label{eq:expert_dominance}
\end{equation}
and the inequality is independent of the physical driving quality of $\pi$.
\end{proposition}
\begin{proof}
Expert replay gives $p^{\pi_{\star}}_k=p^{\star}_k$ and therefore zero deviation at every step. Thus $\phi_r=0$ and, by \eqref{eq:rule_compliance}, $c_r(\pi_{\star},\omega)=100$. Since compliance cannot exceed $100$, \eqref{eq:expert_dominance} follows independently of the executed driving quality.
\end{proof}

\begin{corollary}[Aggregate imitation premium]
\label{cor:premium}
Let $C(\pi,\omega)=\sum_r w_r c_r(\pi,\omega)$. The aggregate gap of any policy $\pi$ relative to expert replay decomposes \emph{exactly} as
\begin{equation}
\begin{aligned}
C(\pi_{\star},\omega)-C(\pi,\omega)=\;&\underbrace{\sum_{r\in\mathcal{R}_{\log}}\!w_r\bigl(100-c_r(\pi,\omega)\bigr)}_{\Delta_{\mathrm{imit}}(\pi,\omega)\;\ge\;0}\\[2pt]
&+\underbrace{\sum_{r\in\mathcal{R}_{\mathrm{ind}}}\!w_r\bigl(c_r(\pi_{\star},\omega)-c_r(\pi,\omega)\bigr)}_{\delta_{\mathrm{ind}}(\pi,\omega)}.
\end{aligned}
\label{eq:premium}
\end{equation}
The first term, the \emph{imitation premium} $\Delta_{\mathrm{imit}}\ge 0$, is paid purely for diverging from the log; the second, $\delta_{\mathrm{ind}}$, is the genuine driving-quality difference and may take either sign. Two consequences follow. A non-imitative planner that matches the expert on the log-independent rules ($\delta_{\mathrm{ind}}=0$) is reported worse by exactly $\Delta_{\mathrm{imit}}$. And a planner that drives \emph{better} ($\delta_{\mathrm{ind}}<0$) is still ranked below the expert whenever $\Delta_{\mathrm{imit}}>-\delta_{\mathrm{ind}}$: the premium can mask superior driving.
\end{corollary}
\begin{proof}
Partition $\mathcal{R}=\mathcal{R}_{\log}\cup\mathcal{R}_{\mathrm{ind}}$, split the sum, and apply Prop.~\ref{prop:dominance} ($c_r(\pi_\star,\omega)=100$ for $r\in\mathcal{R}_{\log}$) to the first group. The decomposition holds termwise; no sign assumption on $\delta_{\mathrm{ind}}$ is used.
\end{proof}

\begin{proposition}[Non-identifiability of driving quality]
\label{prop:nonident}
There exist two policies $\pi_1,\pi_2$ with identical log-independent compliance, $c_r(\pi_1,\omega)=c_r(\pi_2,\omega)$ for all $r\in\mathcal{R}_{\mathrm{ind}}$---equally good drivers under every quality metric---that a log-coupled metric separates: $c_r(\pi_1,\omega)\neq c_r(\pi_2,\omega)$ for some $r\in\mathcal{R}_{\log}$. Consequently, any aggregate with $w_r>0$ on a log-coupled rule is \emph{not} a function of driving quality alone.
\end{proposition}
\begin{proof}
One scenario suffices. Take a straight, two-lane, drivable segment with no other agent in range of any interaction rule, and a logged trajectory $\xi^{\star}_\omega$ centered in lane~1. Let $\pi_1$ reproduce $\xi^{\star}_\omega$, and let $\pi_2$ run the \emph{identical} longitudinal motion---same speed, heading, acceleration, jerk---on a rigidly translated path centered in lane~2, with the offset chosen to exceed some log-coupled threshold $\tau_r$. The two motions differ only by a constant lateral shift, so every \emph{controller-defined} rule (comfort, jerk, lateral acceleration) reads the same on both; no agent is in range, so every \emph{interaction-defined} rule is vacuously satisfied by both; and both stay on the drivable surface under the same limit, so every \emph{map-defined} rule agrees. Thus $c_r(\pi_1,\omega)=c_r(\pi_2,\omega)$ for all $r\in\mathcal{R}_{\mathrm{ind}}$. But the log-coupled predicate fires for $\pi_2$ and not $\pi_1$, so the two differ there. Any aggregate with $w_r>0$ on that rule therefore separates two policies of identical driving quality. The case studies of Sec.~\ref{subsec:cases} realize the same separation approximately on real roads; the empty-road construction makes it exact.
\end{proof}

These results motivate a fairness desideratum for quality metrics.

\begin{definition}
\label{def:invariance}
\emph{Imitation invariance.} A metric is imitation-invariant if its value is unchanged when the reference log $\xi^{\star}_\omega$ is replaced by any other dynamically feasible, rule-satisfying trajectory while the executed driving $D^{\pi}_\omega$ is held fixed.
\end{definition}

Log-independent metrics are imitation-invariant by Def.~\ref{def:logindep}; log-coupled metrics violate it. In Goodhart terms~\citep{manheim_goodhart_2018}, a log-coupled rule optimizes resemblance to one recorded human rather than the underlying objective. The protocol therefore prescribes how W-SQP is scored in Sec.~\ref{sec:results}: report compliance over $\mathcal{R}_{\mathrm{inv}}$ separately from the imitation score over $\mathcal{R}_{\log}$, so driving quality is not conflated with deliberate divergence from the log. Whether the two are in fact separable in practice---$\Delta_{\mathrm{imit}}$ large while the log-independent deficit is null---is the empirical question the experiments settle.

\subsection{Statistical methodology}
\label{subsec:stats}
With $n=150$ paired scenarios under identical configurations, the deterministic simulator makes the ego policy the only configured difference within each pair. Compliance differences are non-normal because many values concentrate near $100\%$, so the Wilcoxon signed-rank test is primary and the paired $t$-test is secondary. Across the $25$ per-rule hypotheses, Benjamini--Hochberg correction controls the false-discovery rate at $\alpha=0.05$~\citep{benjamini_hochberg_1995}. For the contrast between $\mathcal{R}_{\log}$ and $\mathcal{R}_{\mathrm{inv}}$, we report Cohen's $d$, bootstrap $95\%$ confidence intervals for group means, and the \emph{a priori} sample size for $80\%$ power.

\section{Experimental Setup}
\label{sec:setup}
\textbf{Evaluation environment and scope.} The validation is simulation-based. Its relevance to urban autonomous driving comes from four properties: (i) every scenario uses road geometry, actor trajectories, footprints, and speeds from WOMD logs~\citep{ettinger_womd_iccv_2021}; (ii) the ego follows a kinematic-bicycle model with bounded acceleration, steering, lateral acceleration, and jerk (Table~\ref{tab:solver_params}); (iii) non-ego agents respond through the IDM behavior model~\citep{waymax_agents_doc} rather than replaying fixed logs; and (iv) the same controller implementation is timed during $10$\,Hz closed-loop evaluation. Important deployment effects remain absent: the simulator provides ground-truth state and omits sensing and estimation errors, actuation latency, and low-level tracking error. The results, therefore, characterize closed-loop simulation performance and workstation runtime, not on-road transfer or embedded deployment; hardware-in-the-loop and vehicle tests remain necessary (Sec.~\ref{sec:conclusion}).

Every reported number is traceable to a fixed scenario, a deterministic simulator, and evaluator output on executed states. The pipeline separates three components: (i)~a Waymax~\citep{gulino_waymax_2023} wrapper instantiating WOMD~\citep{ettinger_womd_iccv_2021} scenarios at $10$\,Hz; (ii)~an ego policy (expert replay, W-SQP of Sec.~\ref{sec:controller}, or a baseline controller); and (iii)~the $25$-rule evaluator of Sec.~\ref{sec:fair}, consuming only executed states. At each step the loop queries $a_k=\pi(\mathcal{S}_k)$, advances to $\mathcal{S}_{k+1}$ with reactive IDM non-ego agents~\citep{waymax_agents_doc}, and evaluates all monitors. For W-SQP the NLP is rebuilt once in CasADi~\citep{casadi_2019} and re-solved every step via IPOPT~\citep{ipopt_2006}; solver settings and cost weights are in Table~\ref{tab:solver_params}, fixed across all scenarios.

\textbf{Scenario set.} We use $n=150$ WOMD scenarios drawn from shards $0$--$39$ by a complexity-ranked scanner. The ranking combines SDC path length and heading change, agent count and diversity (with additional weight for vulnerable road users), and traffic-control presence. Each scenario lasts approximately $18.7$\,s at $10$\,Hz. This selection favors interaction-rich scenes; Section~\ref{sec:results} repeats the decomposition on a uniformly random sample. Each scenario runs paired expert-replay and W-SQP rollouts under the same NPC configuration; determinism makes the ego policy the only configured difference between a pair. The dataset yields $28{,}131$ NLP solves (approximately $187$ per scenario); $0.58\%$ use the open-loop braking fallback after a solver failure.

\textbf{Exogenous interaction density.} The density falsification control (Sec.~\ref{sec:results}) needs a difficulty measure that is a property of the \emph{scene}, not the planner. At each step we count the other logged agents within radius $R$ of the \emph{logged} ego position and average over the rollout. Computed from the logs, it is identical for both policies---unlike planner-derived proxies such as solve time, which are endogenous and would make any density--compliance correlation circular. We report $R\in\{20,30,50\}$\,m.

\textbf{Log-informed prediction ablation.} On a $30$-scenario subset we replace W-SQP's constant-velocity agent prediction with each agent's constant velocity taken from its logged displacement over the horizon. This uses future log information, so it is not available online; nor is it a perfect oracle, since logged agents do not react to the altered ego rollout. It tests sensitivity to this one velocity estimate.

\textbf{Baseline controllers.} To evaluate W-SQP as a controller, not only as an evaluation instrument, we compare it against two architecturally distinct non-imitative planners on the same $150$ scenarios and evaluator: a lightweight reactive controller (pure-pursuit $+$ IDM) and a stronger sampling-based planner (PDM-Closed--style). Neither scores nor tracks the logged trajectory, so the comparison isolates control performance from the shared property of not matching the log.

\textbf{Baseline controller: pure-pursuit $+$ IDM.} Lateral control is pure pursuit toward a lookahead point on the route geometry (lookahead $6.0+0.6\,v$\,m); longitudinal control is the Intelligent Driver Model with respect to the nearest in-path lead agent (free-flow speed set to the $75$th percentile of the scene's logged ego speed, $T_{\mathrm{hw}}{=}1.5$\,s, $s_0{=}5$\,m). It uses no optimization, tracks no logged states, and follows the route at its own IDM-chosen speed.

\textbf{Baseline controller: PDM-Closed--style.} The stronger baseline is a PDM-Closed--style planner~\citep{dauner_2023_pdm}, a reimplementation in the spirit of the design that won the nuPlan closed-loop challenge (not a fully leaderboard-faithful port; see the route-corridor caveat below). It follows the proposal-and-select recipe: it generates $15$ candidate policies (three lateral offsets $\times$ five IDM target-speed profiles) along a route corridor, forward-simulates each over the $3$\,s horizon with constant-velocity agents, scores them with a nuPlan-CLS--style rule score (collision and off-road gates $\times$ weighted route-progress, comfort, and speed-limit terms, with \emph{no} log-deviation term), and executes the first control of the best. It is non-imitative---the score never references the logged trajectory---yet competent by construction on the rule set. We use the logged path \emph{geometry} as the route corridor (a valid road; a full roadgraph route is a documented extension) and cap the pace to what the corridor affords, so divergence is emergent from the planner's own speed/gap/offset choices, not imposed.

\textbf{Threshold-sensitivity.} To check that the log-coupled deficit is not an artifact of the deviation thresholds, we log the evaluator's per-step raw deviation for the log-coupled rules and recompute their compliance across a sweep of thresholds $\tau$ offline, for both planners (Sec.~\ref{sec:results} reports the range).

\textbf{Divergence sweep.} To characterize the deficit across a continuum of planners between non-imitative W-SQP and expert replay, we add a single scalar $\beta\in[0,1]$ that blends W-SQP's executed next pose toward the logged pose at $t{+}1$: $\beta{=}0$ is W-SQP, $\beta{\to}1$ approaches expert replay. The blend is applied \emph{in closed loop}---the blended pose is executed and the reactive IDM agents respond to it each step, so the log-independent scores along the sweep are well defined. Because the pull toward the log is re-applied every $0.1$\,s, its effect compounds over the rollout: even a small per-step weight drives the executed trajectory close to the log, which is why by $\beta{=}0.25$ the mean divergence has already collapsed to ${\approx}0$ (Fig.~\ref{fig:divergence_sweep}). Sweeping $\beta$ thus parameterizes ego--log divergence directly.

\textbf{Determinism and re-anchoring.} All data are transformed to an ego-local frame for conditioning and inverted after solving. The low-weight reference $x^{\mathrm{ref}}$ is time-indexed from the log; only if ego drift exceeds $10$\,m does it re-anchor to the closest future log point within a $6$\,s window, preventing the tracking term from racing ahead while still allowing recovery. Per-step speed limits use WOMD lane-type proxies (rule $^{3}r_{0}$). These choices keep the controller competent without making it imitative.

\begin{table}[t]
\centering
\caption{Core W-SQP (NMPC) and IPOPT solver settings, fixed across all $150$ scenarios. The tier penalties $(\rho_1,\dots,\rho_4)$ weight squared slacks whose physical units differ by constraint family, so their magnitudes are not directly comparable across tiers (Sec.~\ref{subsec:probe_artifact}); $Q,R,S$ and the shaping weights are the usual dimensionless NMPC cost weights.}
\label{tab:solver_params}
\footnotesize
\renewcommand{\arraystretch}{1.12}\setlength{\tabcolsep}{3pt}
\begin{tabular}{@{}l l@{}}
\toprule
Parameter & Value \\
\midrule
Timestep $\Delta t$ / Horizon $N$ & $0.1$\,s / $30$ ($3.0$\,s) \\
Replan interval & every step ($0.1$\,s) \\
Wheelbase $L$ / max steer $\delta_{\max}$ & $4.5$\,m / $0.5$\,rad \\
Accel bounds $[a_{\min},a_{\max}]$ & $[-5.0,\,3.0]$\,m/s$^2$ \\
Speed bound $v_{\max}$ & $30.0$\,m/s \\
Collision buffer $d_{\mathrm{safe}}$ & $3.0$\,m \\
Road margin / route dev.\ $d_{\mathrm{route}}$ & $0.5$ / $2.0$\,m \\
Comfort dec./lat.\,acc./jerk & $3.0$/$3.0$\,m/s$^2$, $5.0$\,m/s$^3$ \\
Headway time $T_{\mathrm{hw}}$ & $1.5$\,s \\
TL radius / speed thr. & $30.0$\,m / $0.5$\,m/s \\
$(\rho_1,\rho_2,\rho_3,\rho_4)$ & $(10^8,10^5,10^2,1)$ \\
\midrule
IPOPT tol / acceptable tol & $10^{-4}$ / $10^{-3}$ \\
IPOPT linear solver / warm start & MUMPS / enabled \\
$Q,R,S$ & $\diag(1,5,.5,.5)$, $\diag(.1,.5)$, $\diag(.5,1)$ \\
$\alpha_T$ / $W_{\mathrm{jerk}}$ & $10.0$ / $0.1$ \\
$W_{a_{\mathrm{lat}}}$ / $W_{\dot a_{\mathrm{lat}}}$ & $0.5$ / $0.05$ \\
\bottomrule
\end{tabular}
\end{table}

\section{Results}
\label{sec:results}
We now evaluate W-SQP in closed loop. All results use paired evaluation at $10$\,Hz across $n{=}150$ WOMD scenarios, with compliance computed from executed states. We proceed in five steps. We first place W-SQP as a controller against two non-imitative baselines and a no-tiering ablation (Sec.~\ref{subsec:headtohead}). We then show that its apparent aggregate deficit is almost entirely log-coupled (Sec.~\ref{subsec:decomp}), which makes the naive aggregate unfair and lets a log-independent score flip W-SQP's verdict from loss to parity (Sec.~\ref{subsec:instrument}). Finally, we rule out the leading alternative explanations (Sec.~\ref{subsec:falsify}) and illustrate the mechanism on two case studies (Sec.~\ref{subsec:cases}).

\subsection{Control performance against baselines}
\label{subsec:headtohead}
Before analyzing the compliance \emph{score}, we place W-SQP as a controller against the two non-imitative baselines of Sec.~\ref{sec:setup} (reactive PP+IDM and sampling-based PDM-Closed--style) and a no-tiering ablation, on the same $150$ scenarios and the same evaluator. We read control quality off the \emph{log-independent} rules---the safety, regulatory, and interaction rules that score driving quality without reference to the human log. On this group W-SQP attains $93.6\%$ compliance, matching expert replay ($93.7\%$; no systematic group-level deficit, Sec.~\ref{subsec:decomp}) and leading both non-imitative baselines---PP+IDM $91.8\%$ and PDM-Closed $91.5\%$---by roughly $2\pp$. Among controllers that do not track the log, W-SQP is the one whose safety and regulatory behavior comes closest to the human's.

On the \emph{log-coupled} (imitation) axis every non-imitative controller necessarily falls below the expert's $100\%$, but the spread is enormous---W-SQP $89.9\%$, PDM-Closed $67.9\%$, PP+IDM $33.0\%$---and it orders the controllers by how much they happen to resemble the recorded human, not by how safely they drive. W-SQP's comparatively high log-coupled score is itself in part an artifact of its soft route corridor (Eq.~\ref{eq:route}, $d_{\mathrm{route}}{=}2$\,m at Tier-2 weight) and low-weight tracking reference, which supply the route \emph{intent} that WOMD provides no other way to specify; a corridor-free planner of equal driving quality would diverge more and score lower. The $-10.1\pp$ log-coupled deficit is therefore a conservative \emph{lower} bound on the imitation penalty.

This contrast---a tight log-independent ranking beside a $57\pp$-wide log-coupled spread---is the confound in miniature. Fig.~\ref{fig:cep_plane} places every planner in the plane of the two axes. Driving quality (horizontal) separates the planners by only ${\approx}2\pp$; imitation (vertical) spreads them by ${\approx}57\pp$; and the expert sits on the $y{=}100$ dominance boundary of Prop.~\ref{prop:dominance}. A \emph{fair} score reads the horizontal axis alone, so its iso-lines are vertical. The \emph{naive aggregate}'s iso-lines are instead tilted by the $16{:}2$ group sizes, so it ranks a two-dimensional reality by projecting it obliquely---and imitation leaks into the ordering.

\begin{figure}[t]
\centering
\includegraphics[width=\columnwidth]{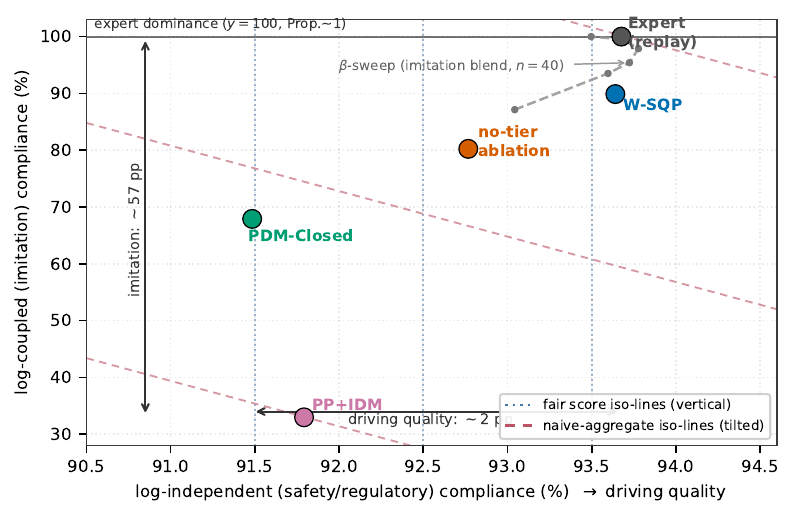}
\caption{The planner-space compliance plane. $x$: log-independent (safety/regulatory) compliance $=$ driving quality; $y$: log-coupled (imitation) compliance. Fair-score iso-lines are vertical (dotted); naive-aggregate iso-lines are tilted (dashed); $y{=}100$ is the expert-dominance boundary (Prop.~\ref{prop:dominance}). The $\beta$-sweep (grey, $n{=}40$ subset) climbs from W-SQP toward the expert corner as the executed pose is blended toward the log.}
\label{fig:cep_plane}
\end{figure}

\textbf{The tiering carries the safety compliance.} The paper's central design claim is that the priority-biased \emph{tiering}, not soft-constraint relaxation on its own, is what protects the high-priority rules. Three experiments support it. First, we remove the tiering entirely, setting all four slack weights equal to their geometric mean ($\bar\rho{=}10^{4}$) and leaving the controller otherwise unchanged. The log-independent deficit grows from $-0.04\pp$ to $-0.91\pp$---an order of magnitude larger---the log-coupled deficit nearly doubles from $-10.1\pp$ to $-19.8\pp$, and overall compliance falls from $93.4\%$ to $91.7\%$. The flat program solves at the same rate, so this is the priority structure at work, not solver behavior.

The flat weight, however, changes two things at once: it removes the \emph{ordering} between tiers and also lowers the overall \emph{stiffness} of the top tiers. The remaining two experiments separate these. We first sweep the priority ratio $r$ between adjacent tiers while holding the geometric-mean penalty fixed at $10^4$, so total stiffness is constant, and only the ordering varies ($r{=}1$ is the flat ablation, $r{=}10^3$ the paper's separation). As ordering strengthens on a $40$-scenario subset, the safety/regulatory deficit improves monotonically from $-1.08\pp$ toward $-0.68\pp$ and the log-coupled deficit from $-17.5$ toward $-13.4\pp$, while median solve time stays ${\approx}30$\,ms throughout (Fig.~\ref{fig:cep_tier_sweep}). We then check the opposite extreme: raising all four weights uniformly to $10^8$---maximally stiff but unordered---degrades the safety/regulatory group \emph{further} on the full set ($-1.03\pp$, versus $-0.91$ at flat $10^4$ and $-0.04$ tiered), with zero solver failures. Stiffer constraints without ordering are worse, not better; ordering is doing the work.

\begin{figure}[t]
\centering
\includegraphics[width=\columnwidth]{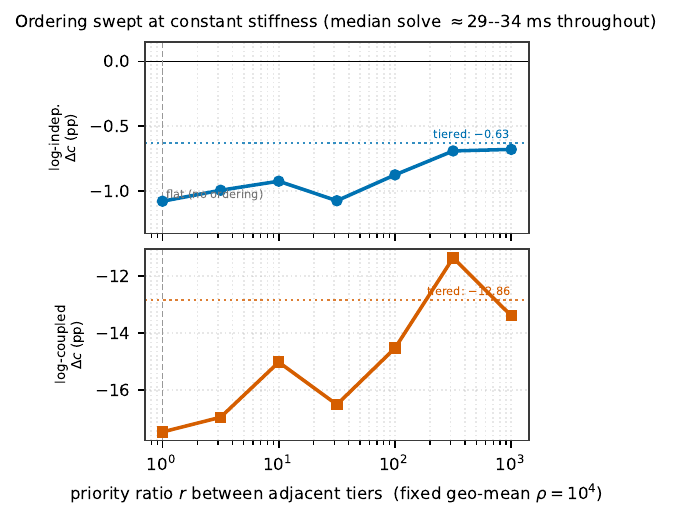}
\caption{Tier-separation mechanism: sweeping the priority ratio $r$ between adjacent tiers at \emph{fixed} geometric-mean stiffness ($\rho{=}10^4$), $40$-scenario subset. As ordering strengthens ($r{=}1$ flat $\to$ $r{=}10^3$), both the log-independent (top) and log-coupled (bottom) deficits improve toward the tiered baseline (dotted) while median solve time stays ${\approx}30$\,ms; the improvement is from ordering, not stiffness.}
\label{fig:cep_tier_sweep}
\end{figure}

The priority bias is also visible as an \emph{auditable, fleet-level} property (Fig.~\ref{fig:cep_audit_panel}): logging each surrogate's residual normalized by its own threshold, the safety/regulatory surrogates are relaxed on only a small fraction of steps (collision $2.5\%$, road-edge $4.9\%$, speed $1.6\%$) while the low-priority route surrogate absorbs most of the relaxation (${\approx}17\%$)---exactly the ordering the tiering is designed to induce, measured across the fleet.

\begin{figure*}[t]
\centering
\includegraphics[width=\textwidth]{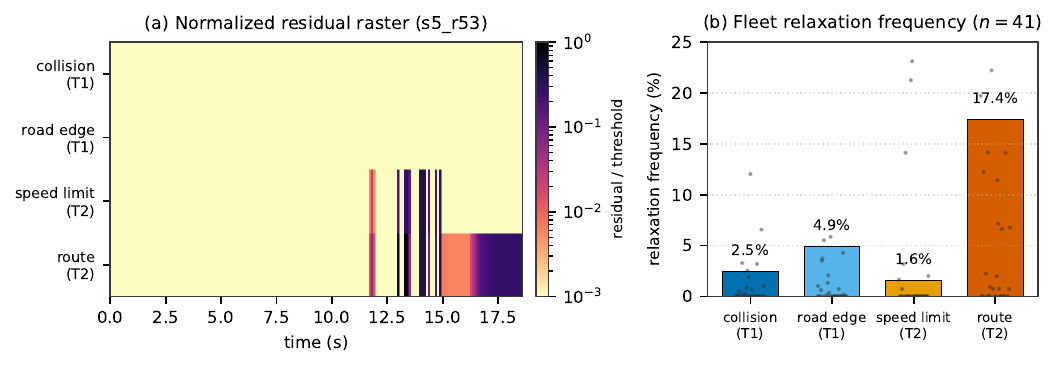}
\caption{Auditability as a measured property. \textbf{(a)}~Threshold-normalized per-surrogate residual raster for one scenario---every cell is a residual divided by its own threshold, so the single colorbar is meaningful across rows: the safety surrogates stay inactive while the low-priority route surrogate carries the late-scene relaxation. \textbf{(b)}~Per-surrogate relaxation frequency across the audited fleet, with per-scenario spread.}
\label{fig:cep_audit_panel}
\end{figure*}

\subsection{The compliance gap is log-coupled}
\label{subsec:decomp}
Fig.~\ref{fig:decomp_groups} shows the mean paired difference $\overline{\Delta c}=c(\pi_{\mathrm{mpc}})-c(\pi_{\star})$ by taxonomy group. The result is stark. The two \emph{log-coupled} rules drop by $\overline{\Delta c}_{\log}=-10.1\pp$, whereas the sixteen \emph{log-independent} safety/regulatory rules show no systematic group-level deficit ($\overline{\Delta c}_{\mathrm{inv}}=-0.04\pp$); the controller-defined group lies between ($-0.75\pp$). We state the second finding as a bound rather than as parity: the bootstrap $95\%$ CI, $[-1.07,+0.52]\pp$, places any systematic safety/regulatory deficit below $1.1\pp$---an order of magnitude under the log-coupled deficit, whose CI is $[-14.0,-6.6]\pp$. The two intervals do not overlap.

We rest the inference on those non-overlapping CIs rather than on a single effect-size number. Cohen's $d=-0.62$ between the two groups is reported as descriptive only, because the groups have very different intrinsic variance ($\mathcal{R}_{\log}$ averages two rules, $\mathcal{R}_{\mathrm{inv}}$ sixteen). For context, an effect of this size needs $n{\approx}23$ scenarios for $80\%$ power, or $n{\approx}42$ under a more conservative two-sample assumption; at $n{=}150$ the study is amply powered. Across scenarios the log-coupled deficit is heavy-tailed---its mean dragged down by a minority of high-divergence scenes---while the log-independent difference stays tight around zero (Fig.~\ref{fig:decomp_groups}).

This is Corollary~\ref{cor:premium} made concrete. Because W-SQP matches expert replay on the log-independent rules ($\delta_{\mathrm{ind}}\approx0$), the entire aggregate gap collapses to the imitation premium---here $\widehat{\Delta}_{\mathrm{imit}}=10.1\pp$. W-SQP is not a worse driver; it is a \emph{different} driver, penalized only for not being the recorded human (Prop.~\ref{prop:dominance}, Prop.~\ref{prop:nonident}).

The per-rule view (Table~\ref{tab:per_rule_compliance}) confirms the deficit is not diffuse. The two largest negative deltas are path tracking ($^{0}r_{0}$, $-10.7\pp$) and route adherence ($^{2}r_{2}$, $-9.5\pp$)---both at $100\%$ for expert replay by construction. The log-independent rules move in both directions by small amounts (headway $+2.6$, clearance $+1.9$, against speed limit $-2.3$, drivable surface $-2.1$) and cancel at the group level, which is why the claim is a group-level one, not per-rule parity.

The small genuine regressions are not spread across the dataset either. The three largest significant safety/regulatory deficits---non-traversable surface, drivable surface, and speed limit (Table~\ref{tab:two_deficits})---concentrate in the same high-divergence scenarios as the log-coupled penalty. They correlate with ego--log divergence at $r{=}{-}0.90$, averaging $-4.8\pp$ in the high-divergence half of scenarios but exactly $0.0\pp$ in the low-divergence half. These are recovery transients in the hardest scenes, not a broad safety regression: in half the scenarios the safety/regulatory deficit simply vanishes. Benjamini--Hochberg correction flags $10$ of $25$ rules, the two log-coupled rules carrying the smallest $q$-values.

\begin{table}[t]
\centering
\caption{Two kinds of deficit, separated. \textbf{(A)}~the \emph{log-coupled imitation penalty}: rules defined as deviation from the human log, which expert replay satisfies at $100\%$ by construction. \textbf{(B)}~absolute regressions on genuine log-independent safety/regulatory rules (BH-significant, negative), offset by improvements elsewhere so the group nets to $-0.04\pp$; these concentrate in the highest-divergence scenes (Sec.~\ref{subsec:decomp}). $\Delta$: paired mean (W-SQP$-$expert), pp.}
\label{tab:two_deficits}
\small\renewcommand{\arraystretch}{1.15}\setlength{\tabcolsep}{5pt}
\begin{tabular}{l l r}
\toprule
Rule & Group & $\Delta$ (pp) \\
\midrule
\multicolumn{3}{l}{\emph{(A) Log-coupled imitation penalty}}\\
Path tracking ($^{0}r_{0}$)   & log-coupled & $-10.7$ \\
Route adherence ($^{2}r_{2}$) & log-coupled & $-9.5$ \\
\midrule
\multicolumn{3}{l}{\emph{(B) Absolute safety/regulatory regressions}}\\
Non-traversable surface ($^{9}r_{1}$) & log-indep. & $-2.9$ \\
Speed limit ($^{3}r_{0}$)             & log-indep. & $-2.3$ \\
Drivable surface ($^{7}r_{0}$)        & log-indep. & $-2.1$ \\
Bike-lane encroachment ($^{10}r_{5}$) & log-indep. & $-1.0$ \\
\bottomrule
\end{tabular}
\end{table}

\begin{figure}[t]
\centering
\includegraphics[width=\columnwidth]{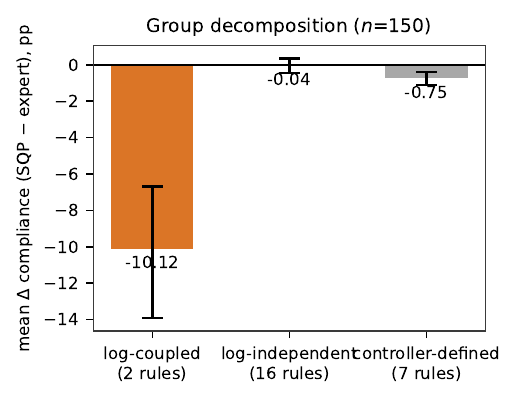}
\caption{Mean paired compliance difference (W-SQP $-$ expert) by taxonomy group, bootstrap $95\%$ CIs ($n{=}150$). The gap is confined to the two log-coupled rules; the log-independent group shows no systematic deficit and has a tight per-scenario distribution around zero, whereas the log-coupled group is heavy-tailed. Individual rules move both ways (Table~\ref{tab:two_deficits}).}
\label{fig:decomp_groups}
\end{figure}

\begin{table*}[t]
\centering
\caption{Per-rule compliance (\%) over 150 paired WOMD scenarios. \textbf{Grp}: taxonomy group (Sec.~\ref{sec:taxonomy}); \textbf{Exp}/\textbf{SQP}: mean compliance for expert replay / W-SQP; $\Delta$: paired mean (SQP$-$Exp), pp; \textbf{$q$}: Benjamini--Hochberg FDR-adjusted Wilcoxon $q$-value ($m{=}25$; \textbf{bold} if $q{\le}0.05$). A `--' marks a rule never violated by either policy (zero variance, no test); such rules still enter the $m{=}25$ family, which only makes the correction more conservative.}
\label{tab:per_rule_compliance}
\small\renewcommand{\arraystretch}{1.0}\setlength{\tabcolsep}{4pt}
\begin{tabular}{c l l l r r r r}
\toprule
\textbf{Lvl} & \textbf{Rule} & \textbf{Name} & \textbf{Grp} & \textbf{Exp} & \textbf{SQP} & \boldmath$\Delta$ & \textbf{$q$} \\
\midrule
10 & $^{10}r_{0}$ & VRU collision avoidance & log-indep. & 98.90 & 99.10 & $+0.20$ & 0.867 \\
10 & $^{10}r_{4}$ & Cyclist passing distance & log-indep. & 99.26 & 98.91 & $-0.35$ & 0.304 \\
10 & $^{10}r_{5}$ & Bike lane encroachment & log-indep. & 99.04 & 98.08 & $-0.96$ & \textbf{0.008} \\
\midrule
9 & $^{9}r_{0}$ & Vehicle collision avoidance & log-indep. & 98.01 & 98.53 & $+0.52$ & 0.208 \\
9 & $^{9}r_{1}$ & Non-traversable surface & log-indep. & 99.98 & 97.12 & $-2.85$ & \textbf{0.002} \\
\midrule
8 & $^{8}r_{0}$ & Stop sign compliance & log-indep. & 99.07 & 98.79 & $-0.28$ & 0.222 \\
8 & $^{8}r_{1}$ & Crosswalk yield & log-indep. & 93.48 & 93.89 & $+0.41$ & 0.362 \\
\midrule
7 & $^{7}r_{0}$ & Drivable surface & log-indep. & 100.00 & 97.94 & $-2.06$ & \textbf{0.008} \\
7 & $^{7}r_{1}$ & Traffic light compliance & log-indep. & 97.82 & 98.11 & $+0.29$ & 0.660 \\
7 & $^{7}r_{9}$ & Gore area avoidance & log-indep. & 96.73 & 98.33 & $+1.60$ & 0.362 \\
\midrule
6 & $^{6}r_{4}$ & Elephant racing & log-indep. & 100.00 & 100.00 & $+0.00$ & -- \\
\midrule
3 & $^{3}r_{0}$ & Speed limit & log-indep. & 99.50 & 97.25 & $-2.25$ & \textbf{0.008} \\
3 & $^{3}r_{3}$ & Following headway & log-indep. & 85.93 & 88.53 & $+2.59$ & \textbf{0.005} \\
3 & $^{3}r_{4}$ & Lateral clearance & log-indep. & 54.99 & 56.85 & $+1.85$ & 0.304 \\
3 & $^{3}r_{6}$ & Lane intrusion & log-indep. & 81.04 & 80.74 & $-0.30$ & 0.867 \\
\midrule
2 & $^{2}r_{0}$ & Time window compliance & controller & 64.97 & 66.27 & $+1.30$ & \textbf{0.005} \\
2 & $^{2}r_{2}$ & Route adherence & log-coupled & 100.00 & 90.50 & $-9.50$ & \textbf{$<$0.001} \\
\midrule
1 & $^{1}r_{0}$ & Yield to priority & log-indep. & 95.07 & 96.07 & $+1.00$ & 0.055 \\
1 & $^{1}r_{11}$ & Lateral acceleration & controller & 99.96 & 99.98 & $+0.02$ & 0.867 \\
1 & $^{1}r_{3}$ & Turn signal advisory & controller & 100.00 & 100.00 & $+0.00$ & -- \\
1 & $^{1}r_{4}$ & Lane selection & controller & 100.00 & 100.00 & $+0.00$ & -- \\
1 & $^{1}r_{9}$ & Jerk limit & controller & 96.31 & 92.57 & $-3.74$ & \textbf{0.002} \\
\midrule
0 & $^{0}r_{0}$ & Path tracking & log-coupled & 100.00 & 89.27 & $-10.73$ & \textbf{$<$0.001} \\
0 & $^{0}r_{2}$ & Longitudinal comfort & controller & 99.85 & 97.02 & $-2.83$ & \textbf{$<$0.001} \\
0 & $^{0}r_{3}$ & Lateral comfort & controller & 99.95 & 99.95 & $+0.00$ & 1.000 \\
\bottomrule
\end{tabular}
\end{table*}

\subsection{The aggregate is unfair; the log-independent score repairs it}
\label{subsec:instrument}
The decomposition has a direct consequence for ranking. Fig.~\ref{fig:verdict_flip} plots, per scenario, the naive aggregate $\Delta c$ (all $25$ rules) against the log-independent $\Delta c$. Almost every point lies below the parity line: the aggregate scores W-SQP harsher than the fair score, and the offset is the log-coupled contribution. Converting to verdicts with a $\pm0.5\pp$ dead-band, the naive aggregate records W-SQP as losing $44/150$ scenarios and winning $32$ (mean $-1.04\pp$); under the log-independent score the verdict is balanced---$30$ losses, $34$ wins (mean $-0.04\pp$)---with $14$ scenarios flipped out of ``loss.'' A benchmark using the aggregate would report a net-negative, misleading ranking for a planner whose safety and regulatory compliance is, at the dataset level, statistically indistinguishable from the human's (Table~\ref{tab:two_deficits} details the localized exceptions). Reporting the log-independent score, with the log-coupled score separately as an imitation metric, removes the bias with no change to the simulator or agents---though, as Sec.~\ref{sec:discussion} discusses, it presupposes an agreed rule classification and calibrated thresholds.

\textbf{The confound is not an artifact of equal weighting.} One might object that an unweighted mean over $25$ heterogeneous rules is a straw aggregator, since rulebooks are priority-ordered and the log-coupled rules sit at the \emph{lowest} levels ($0$ and $2$). Re-weighting the aggregate by priority ($w_r\propto(\ell_r{+}1)^p$, $\ell_r$ the rule's level) shrinks but does not remove the gap: the mean W-SQP$-$expert gap moves from $-1.04\pp$ (equal weights) to $-0.52\pp$ ($p{=}1$) and $-0.41\pp$ ($p{=}2$), and the two lowest-priority log-coupled rules still account for $78\%$ of the equal-weight gap. By Corollary~\ref{cor:premium} the imitation premium is positive for \emph{any} weighting with $w_r>0$ on a log-coupled rule; only excluding those rules removes it cleanly.

\begin{figure}[t]
\centering
\includegraphics[width=\columnwidth]{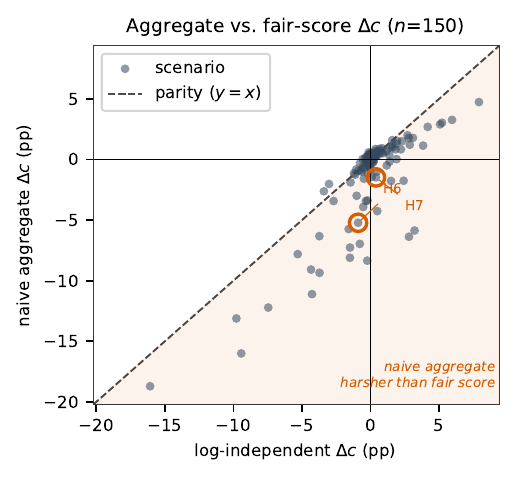}
\caption{Per-scenario naive aggregate $\Delta c$ vs.\ the log-independent (fair) $\Delta c$; points below parity are scenarios the aggregate penalizes beyond the fair score. Under a $\pm0.5\pp$ dead-band the naive score records W-SQP at $44$ losses / $32$ wins, the fair score at $30$ / $34$---flipping $14$ scenarios out of ``loss.'' Case-study scenarios circled.}
\label{fig:verdict_flip}
\end{figure}

\subsection{Falsification controls}
\label{subsec:falsify}
\textbf{The deficit is a predictable function of divergence.} The magnitudes above are for two specific planners; what would a \emph{less-divergent} planner---the kind that tops a leaderboard---incur? To answer, we build a continuum between non-imitative W-SQP ($\beta{=}0$) and expert replay ($\beta{\to}1$) by blending W-SQP's executed pose toward the logged pose with weight $\beta$, and trace the deficit against mean ego--log divergence (Fig.~\ref{fig:divergence_sweep}). The blend is a synthetic interpolation used only to parameterize divergence---the blended poses are not driven through the bicycle model---so the curve characterizes divergence versus deficit, not a family of realizable plans. The relationship is monotone and saturating: the deficit falls smoothly toward zero as divergence shrinks, vanishing only in the imitative limit, while log-independent compliance stays at parity throughout ($-0.6$ to $+0.1\pp$).

Two things follow. A planner that stays closer to the log incurs a smaller deficit, but by Prop.~\ref{prop:dominance} never a zero one. And the mapping is not universal across planners: PDM-Closed sits \emph{above} this curve, with a larger deficit ($-32\pp$) at the same mean divergence ($2.2$\,m). The reason is that the deficit depends on the \emph{distribution} of divergence, not its mean---bursty excursions cross the threshold more often than smoothly spread deviations of equal average. Divergence predicts the deficit within a planner family but does not reduce it to one number across families.

\begin{figure}[t]
\centering
\includegraphics[width=\columnwidth]{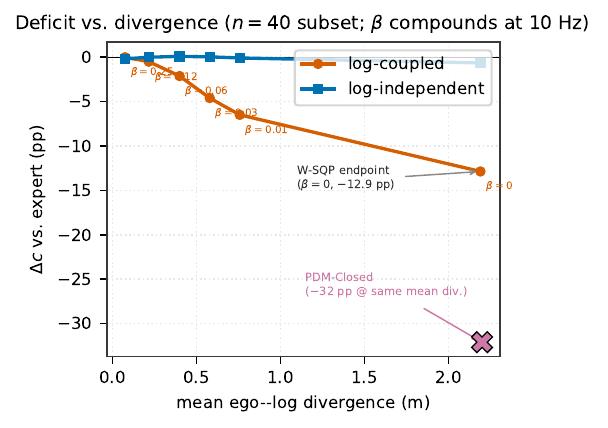}
\caption{Log-coupled and log-independent $\Delta c$ vs.\ mean ego--log divergence ($n{=}40$ subset), sweeping the imitation-blend weight from $\beta{=}0$ (W-SQP, divergence $2.2$\,m) to $\beta{=}0.25$ (divergence ${\approx}0$). The log-coupled deficit is a monotone, saturating function of divergence while log-independent compliance stays at parity; the PDM-Closed marker ($\times$) sits above the curve.}
\label{fig:divergence_sweep}
\end{figure}

\textbf{It is not interaction density.} If the deficit reflected difficulty in crowds, it should grow with interaction density. It does the opposite. Regressing the log-coupled $\Delta c$ on the exogenous scene density of Sec.~\ref{sec:setup} gives a \emph{positive} slope ($+0.45\pp$/agent at $R{=}30$\,m, $p{=}0.020$), while the log-independent slope is flat (Fig.~\ref{fig:density_falsify}). A mediation analysis explains the sign: divergence drives the deficit ($r{=}-0.68$), but dense traffic produces \emph{less} divergence ($r{=}-0.22$), because it funnels the ego onto the only feasible path---close to the human's---whereas sparse scenes leave room for a valid alternative. The only proxy that yields the ``expected'' negative slope is the planner's own solve time, an endogenous difficulty signal and precisely the circular measure an exogenous density is designed to avoid. Since the divergence-to-deficit link is partly tautological, we rest the main claim on the decomposition; this control establishes only the negative.

\textbf{It is not the velocity predictor.} W-SQP predicts other agents at their current constant velocity. Replacing this with a log-informed horizon-average velocity on a $30$-scenario subset barely moves the log-coupled difference ($-10.3\pp\!\to\!-11.5\pp$) and leaves the log-independent group unchanged. The check is limited to $30$ scenarios and to one alternative estimate, so it supports robustness to the velocity model without ruling out prediction effects in general.

\textbf{It is not specific to W-SQP.} We repeat the whole decomposition on the two non-imitative baselines; the question now is not their control quality (Sec.~\ref{subsec:headtohead}) but whether the \emph{confound} itself reproduces. It does, on both (Table~\ref{tab:generality}): the log-coupled deficit is large (PP+IDM $-67.0\pp$; PDM-Closed $-32.1\pp$) and far exceeds the log-independent one in each case. For the baselines, we claim only this reproducibility, not safety parity---their log-independent \emph{means} are negative ($-1.9$ and $-2.2\pp$), so they do drive somewhat less safely than the human on average, and PDM-Closed's route corridor pushes it off-road in $26\%$ of scenarios (a full roadgraph route would fix this). Only W-SQP reaches group-level log-independent parity. But the log-coupled deficit reproduces across three planners that share nothing except being non-imitative, which is the point: the deficit is a property of the metric, not of any one controller.

\textbf{It is not a threshold artifact.} The log-coupled rules fire when deviation from the logged pose exceeds a threshold $\tau$ (nominally $2.4$\,m). Recomputing path-tracking compliance as $\tau$ sweeps over $[1,6]$\,m leaves the deficit large and negative everywhere: $[-26,-8]\pp$ for W-SQP and $[-83,-53]\pp$ for PP+IDM, with expert replay at $100\%$ for every $\tau$ by construction. No reasonable threshold removes the premium.

\textbf{It is not scenario selection.} We rank the main scenarios by complexity, which might correlate with divergence. Re-running the decomposition on a uniformly random sample of $50$ WOMD scenarios preserves the effect and, if anything, sharpens it: $\overline{\Delta c}_{\log}=-16.7\pp$ against $\overline{\Delta c}_{\mathrm{inv}}=-1.0\pp$ (median $0.0\pp$). Selecting by complexity therefore \emph{understates} the deficit rather than manufacturing it, consistent with the finding that sparser scenes admit more divergence.

\begin{table}[t]
\centering
\caption{Generality across planners: group compliance difference vs.\ expert replay ($n{=}150$, pp). The log-coupled $\gg$ log-independent pattern reproduces on a reactive controller and a recognized PDM-Closed--style planner. The log-independent row shows \textbf{mean (median)}: the negative \emph{means} on the baselines are pulled by a minority of high-divergence scenarios, while the \emph{medians} sit near zero---the robustness message.}
\label{tab:generality}
\small\renewcommand{\arraystretch}{1.15}\setlength{\tabcolsep}{5pt}
\begin{tabular}{l r r r}
\toprule
Group & W-SQP & PP+IDM & PDM-C \\
\midrule
log-coupled (2), mean          & $-10.1$ & $-67.0$ & $-32.1$ \\
log-indep.\ (16), mean   & $-0.04$ & $-1.9$  & $-2.2$ \\
\quad\emph{(median)}           & $-0.04$ & $-0.15$ & $-0.17$ \\
controller (7), mean           & $-0.75$ & $+2.5$  & $-1.9$ \\
\bottomrule
\end{tabular}
\end{table}

\begin{figure}[t]
\centering
\includegraphics[width=\columnwidth]{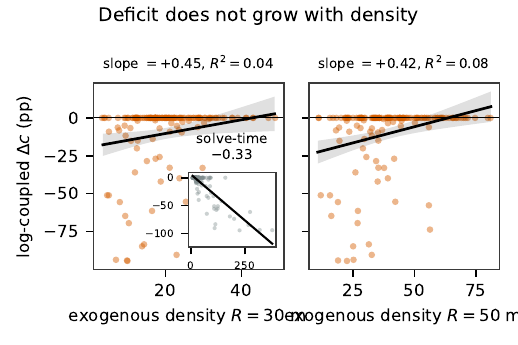}
\caption{Log-coupled $\Delta c$ vs.\ exogenous interaction density at $R{=}30$ and $50$\,m: positive/flat slopes (the deficit does not grow with crowding). Inset: the circular solve-time proxy gives the opposite, misleading slope.}
\label{fig:density_falsify}
\end{figure}

\subsection{Case studies}
\label{subsec:cases}
Two representative scenarios make the mechanism concrete (Table~\ref{tab:heroes}): a stop-controlled junction (H6, Figs.~\ref{fig:hero6}--\ref{fig:hero6_raster}) and dense vehicle traffic (H7, $290$ vehicles). In both, the trajectory overlay shows W-SQP taking a valid, different path from the logged human---a different lane or a different gap --- and the per-rule panel shows the log-coupled rules lit up for W-SQP while its safety and regulatory rules stay close to the expert's. The per-scenario numbers tell the same story: a large log-coupled deficit beside near-parity log-independent compliance (H6 $\Delta_{\mathrm{inv}}{=}{-}0.9$, H7 $+0.4\pp$, both inside the dataset-level CI). These are exactly the instances of Prop.~\ref{prop:dominance}---a valid alternative path scored down purely for not being the recorded human's.

\begin{table}[t]
\centering
\caption{Case-study scenarios (complex, rule-rich scenes). \textbf{Veh}: \# vehicles; \textbf{Brd}: \# rules applicable $>$$10\%$ of steps; $\Delta_{\log}$, $\Delta_{\mathrm{inv}}$: per-scenario group compliance differences (pp; $\Delta_{\mathrm{inv}}\!>\!0$ means W-SQP is \emph{safer} than the human); \textbf{div}: mean ego--log divergence (m).}
\label{tab:heroes}
\small\renewcommand{\arraystretch}{1.1}\setlength{\tabcolsep}{4pt}
\begin{tabular}{l l r c r r r}
\toprule
 & Scenario & Veh & Brd & $\Delta_{\log}$ & $\Delta_{\mathrm{inv}}$ & div \\
\midrule
H6 & Stop-controlled           & 58  & 25 & $-36.0$ & $-0.9$ & $4.6$ \\
H7 & Dense vehicle traffic     & 290 & 24 & $-15.0$ & $+0.4$ & $9.8$ \\
\bottomrule
\end{tabular}
\end{table}

\begin{figure*}[t]
\centering
\begin{subfigure}{0.32\textwidth}\includegraphics[width=\linewidth]{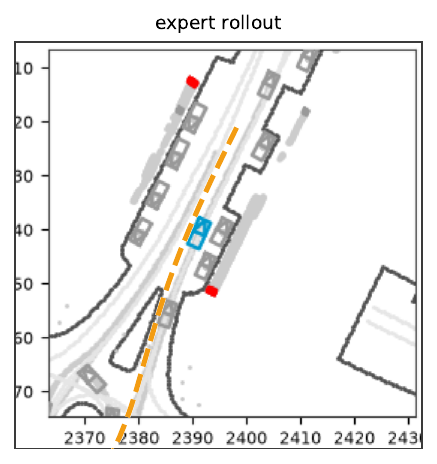}\end{subfigure}\hfill
\begin{subfigure}{0.32\textwidth}\includegraphics[width=\linewidth]{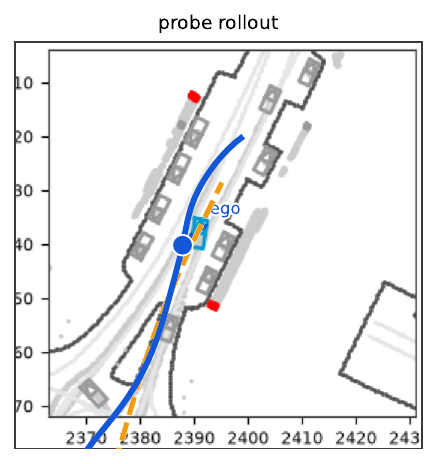}\end{subfigure}\hfill
\begin{subfigure}{0.30\textwidth}\includegraphics[width=\linewidth]{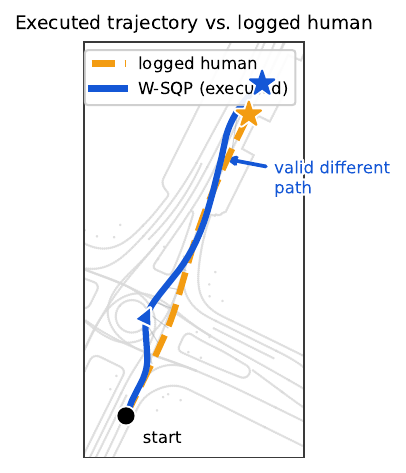}\end{subfigure}
\caption{\textbf{H6 --- Stop-controlled.} Expert vs.\ W-SQP rollouts at maximum divergence (left, centre) and the executed-vs-logged trajectory overlay (right). W-SQP takes a valid alternative line with near-parity safety ($\Delta_{\mathrm{inv}}{=}{-}0.9\pp$) and a $-36.0\pp$ log-coupled deficit; the per-rule breakdown is in Fig.~\ref{fig:hero6_raster}.}
\label{fig:hero6}
\end{figure*}

\begin{figure}[t]
\centering
\includegraphics[width=\columnwidth]{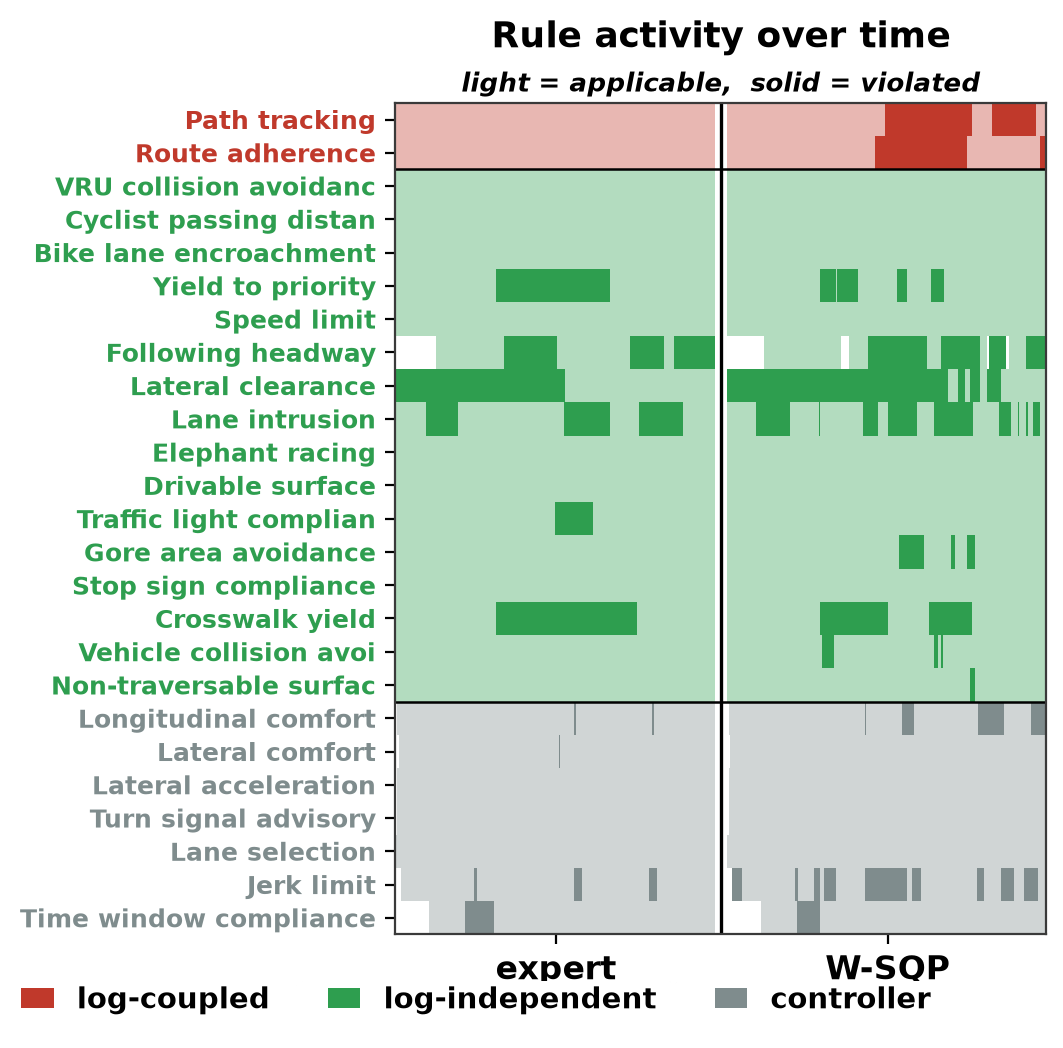}
\caption{\textbf{H6 --- Rule-activity timeline} (expert $\vert$ W-SQP; light = applicable, solid = violated). The log-coupled rules (path tracking, route adherence) are elevated for W-SQP while the safety and regulatory rules stay close to expert replay: the $-36.0\pp$ log-coupled deficit of Fig.~\ref{fig:hero6} is confined to the imitation rules.}
\label{fig:hero6_raster}
\end{figure}

\begin{figure*}[t]
\centering
\begin{subfigure}{0.32\textwidth}\centering\includegraphics[width=\linewidth,keepaspectratio]{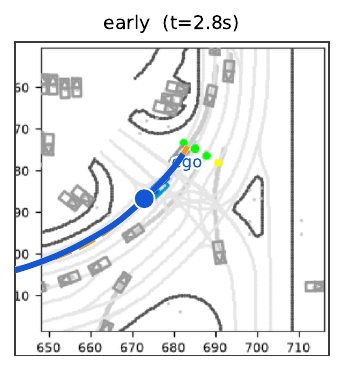}\end{subfigure}\hfill
\begin{subfigure}{0.32\textwidth}\centering\includegraphics[width=\linewidth,keepaspectratio]{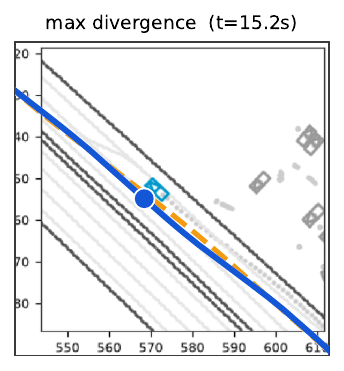}\end{subfigure}\hfill
\begin{subfigure}{0.32\textwidth}\centering\includegraphics[width=\linewidth,keepaspectratio]{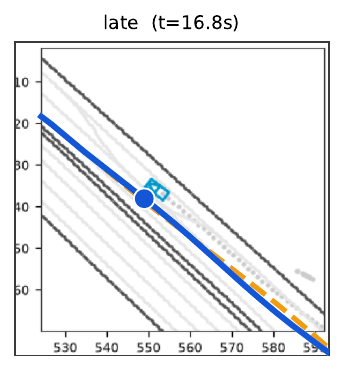}\end{subfigure}

\vspace{3pt}
\includegraphics[width=0.62\textwidth,keepaspectratio]{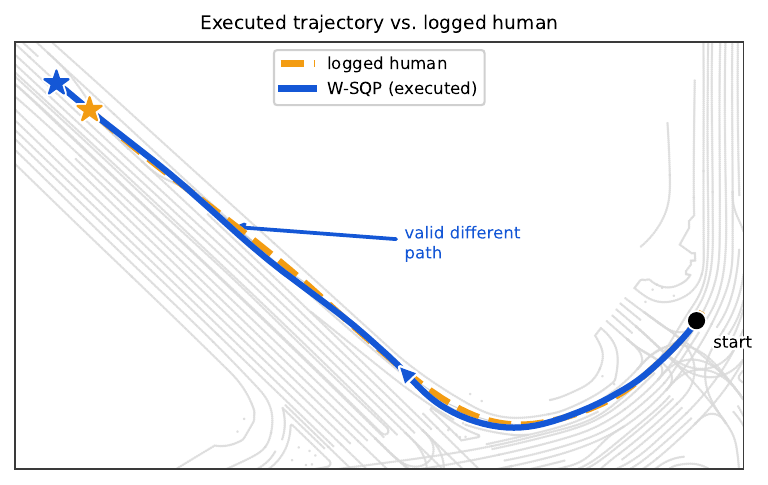}
\caption{\textbf{H7 --- Dense vehicle traffic ($290$ vehicles).} Bird's-eye keyframes (early / max-divergence / late) and the executed-vs-logged overlay. W-SQP accepts a different gap and is marginally \emph{safer} ($\Delta_{\mathrm{inv}}{=}{+}0.4\pp$) while its log-coupled compliance drops $-15.0\pp$; the per-rule mechanism over time is in Fig.~\ref{fig:h7_vignette}.}
\label{fig:hero7}
\end{figure*}

Fig.~\ref{fig:h7_vignette} tracks the same scenario over time. W-SQP's path-tracking deviation climbs to $7.7$\,m and stays above the $2.4$\,m threshold for much of the rollout, so the log-coupled rule fires continuously. Yet over the same interval its physical margins---collision closeness, speed over the limit, headway---track expert replay and stay safe. The controller is penalized for tens of seconds while nothing physical is wrong.

\begin{figure}[t]
\centering
\includegraphics[width=0.9\columnwidth]{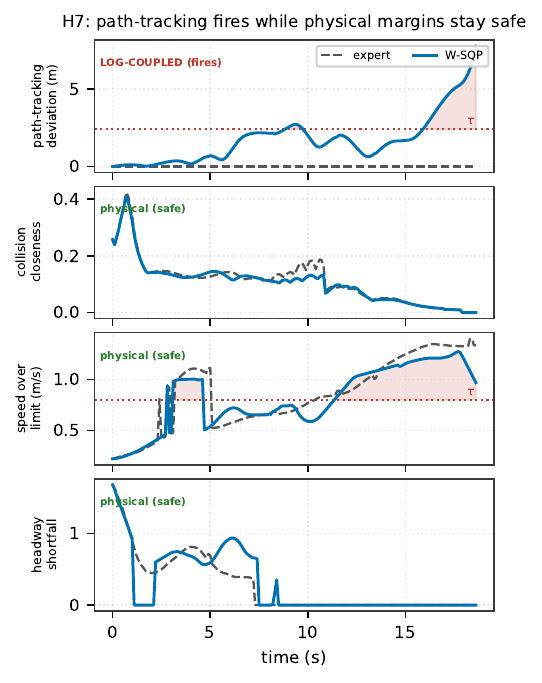}
\caption{Prop.~\ref{prop:dominance} in time (H7, W-SQP solid vs.\ expert dashed; all quantities ``higher $=$ worse''). \textbf{Top:} the log-coupled path-tracking deviation exceeds its threshold $\tau{=}2.4$\,m (red) for much of the rollout. \textbf{Below:} the physical margins (collision closeness, speed-over-limit, headway shortfall) track expert replay and stay safe throughout.}
\label{fig:h7_vignette}
\end{figure}

\subsection{Summary}
A competent non-imitative planner is penalized almost entirely on rules that measure deviation from the human log, while its safety and regulatory compliance is statistically indistinguishable from expert replay. The effect is large; it is neither interaction density (which reverses the sign) nor prediction quality (in our ablation). Two claims should be kept apart. The \emph{structural} premium---expert replay scores the log-coupled rules at $100\%$ by construction, so any non-imitative planner pays a positive premium under any catalog that weights such a rule---holds in general. The \emph{magnitude} ($-10.1\pp$) is specific to this controller and scenario set. An aggregate that includes the log-coupled rules ranks the planner a net loser; the log-independent score ranks it at parity.

\section{Discussion}
\label{sec:discussion}
\subsection{Deployment and practical relevance}
As autonomous systems move toward deployment, planner decisions must be intelligible not only to the developers who tune them but to the safety engineers and investigators who must justify them. Being able to establish, after the fact, which rule was relaxed and by how much is thus an engineering requirement, not a reporting convenience. W-SQP offers three properties toward it. It is computationally tractable on the evaluated workstation---median solve time about $28$\,ms, maximum $104$\,ms under the $90$\,ms limit, with dynamics projection guaranteeing a model-consistent next pose even for a non-converged iterate. It records threshold-normalized residuals for every surrogate, making high-priority relaxations visible for review (Fig.~\ref{fig:cep_audit_panel}). And its tier weights are explicit tuning knobs for the priority bias. None of this amounts to a safety guarantee: the weights do not impose lexicographic priority, the soft surrogates do not certify safety, and deployment still needs embedded profiling and validation with sensing, estimation, and actuation in the loop.

\subsection{A log-independent compliance score}
Our results point to a simple, low-cost remedy. Because the confounded rules form an explicit \emph{a priori} subset, a benchmark can report two numbers instead of one: a \emph{log-independent compliance} score over the safety, regulatory, map-, and interaction-defined rules, and, separately, an \emph{imitation} score over the log-coupled rules. The first measures what closed-loop evaluation is meant to measure---did the planner drive safely and legally?---while the second is reported for transparency rather than folded into the ranking. This needs no change to the simulator or agents; it is a relabeling and re-aggregation of metrics benchmarks already compute. It does presuppose agreement on which rules are quality metrics and which are imitation metrics. We therefore recommend that leaderboards mixing the two kinds keep log-coupled metrics out of the aggregate quality score, and recalibrate the retained thresholds before treating them as absolute compliance measures.

The log-independent score must still reward task completion, or a planner could look safe merely by making no progress. A \emph{map-route progress} term supplies this, defined against the assigned route corridor rather than the log---in the manner of nuPlan's progress-along-route ratio and Waymax's route-following metric~\citep{nuplan_2021,gulino_waymax_2023}. Referenced to route geometry and not to $\xi^{\star}_\omega$, it is log-independent by Def.~\ref{def:logindep}. Computed as executed arc-length along the route over route length, it gives expert replay, W-SQP, and PP+IDM comparable scores ($0.93$, $0.91$, $0.89$) while a do-nothing planner scores near zero. This term is a recommendation only: it is computed here post-hoc and is not part of the validated log-independent score of Sec.~\ref{sec:results}, whose verdict flip rests on the safety and regulatory rules alone.

Applied to existing benchmarks, the taxonomy is a simple audit (Table~\ref{tab:benchmarks}), and mature ones already pass it: nuPlan's closed-loop score contains only interaction-, map-, and controller-defined terms, with pointwise matching confined to the \emph{separate} open-loop score, and Waymax reports log-divergence standalone rather than in any aggregate~\citep{nuplan_2021,gulino_waymax_2023,dauner_2023_pdm}. The confound is not a flaw of these benchmarks but a risk for \emph{any} catalogue that folds a log-coupled term into one ranking---as our own $25$-rule aggregate does. To make the fix turnkey we release a reference implementation that, given per-rule compliance and a per-rule taxonomy label, reports the aggregate, the fair score, and the imitation premium, with the verdict under each.

\begin{table}[t]
\centering
\caption{The taxonomy as a benchmark audit. Mature closed-loop scores keep log-coupled terms out of the ranking aggregate; the confound is a risk only for catalogues that do not.}
\label{tab:benchmarks}
\small\renewcommand{\arraystretch}{1.15}\setlength{\tabcolsep}{4pt}
\begin{tabular}{@{}l c p{3.1cm}@{}}
\toprule
Score & in agg.? & note \\
\midrule
nuPlan CLS      & no          & ADE/FDE confined to OLS \\
nuPlan OLS      & (separate)  & pointwise log-matching only \\
Waymax metrics  & no          & log-divergence reported standalone \\
Our catalogue   & \textbf{yes} & path-tracking $+$ route folded in \\
\bottomrule
\end{tabular}
\end{table}

\subsection{Relation to agent-realism findings}
This reframes part of the score deterioration that learned reactive agents induce on existing benchmarks~\citep{hagedorn_nuplan_reactive_2025,peng_nuplan_r_2025}. As agents become more reactive, the recorded human log becomes a less attainable target, so any log-coupled sub-metric reports growing ``violations'' that reflect divergence from an un-followable log rather than degraded driving. Our per-rule decomposition makes this component separable from genuine compliance change and is complementary to that line of work: realism changes the rollout, metric construction changes what the rollout is scored against, and both must be controlled for fair comparison. We do not claim this accounts for the whole of the reported shifts, which are planner- and scenario-dependent; only that our decomposition isolates one component attributable to metric construction rather than to driving quality.

\subsection{Limitations}
\label{subsec:limitations}
The central threat to validity is external. We establish the decomposition for three non-imitative planners, but the evidence still rests on a single evaluator, a single (IDM) reactive-agent model, and one dataset; the magnitude of the log-coupled gap will vary with each. Our PDM-Closed also uses the log path geometry as its route corridor and occasionally leaves the drivable area on complex geometry ($26\%$ of scenarios, inflating its log-independent mean but not its near-parity median). A full roadgraph route, a second evaluator catalogue, and learned reactive agents are the main future work.

A second asymmetry the taxonomy does \emph{not} address is one of policy \emph{reactivity} rather than metric reference. Expert replay cannot respond to the agents, only they to it. So some of W-SQP's small positive deltas on interaction rules (headway $+2.6$, clearance $+1.9\pp$) may reflect replay's passivity rather than better driving---the reactivity counterpart of the log-coupling confound, this time in W-SQP's favor. Since our headline is parity, not superiority, this does not threaten the result; but a fully symmetric comparison would need a reactive reference rather than a replayed one.

The evaluator calibration itself warrants a caveat. Some log-independent rules flag even the logged human at high rates (lateral clearance at ${\sim}55\%$, following headway at ${\sim}86\%$ for expert replay), largely because compliance~\eqref{eq:rule_compliance} normalizes over all steps rather than applicable ones. The log-independent score is therefore best read as a \emph{relative}, controller-vs-expert measure; individual thresholds would need recalibration before absolute use. Our claim is a paired decomposition, so this does not affect the headline. Finally, the scenarios are complexity-selected and W-SQP relaxes rules softly, so compliance is not a formal safety guarantee; the evaluator remains authoritative.

\section{Conclusion}
\label{sec:conclusion}
We have presented W-SQP, an auditable, rulebook-aware NMPC for autonomous driving. The controller maps nine rule families to four shared-slack tiers in one CasADi/IPOPT nonlinear program. Separated quadratic penalties bias violations toward lower-priority rules while actuation bounds remain hard; they do not provide lexicographic priority or formal safety guarantees. Threshold-normalized per-rule residuals provide an audit record for each control cycle.

W-SQP replans from its executed state at $10$\,Hz. A $90$\,ms CPU-time limit returns IPOPT's best iterate, and dynamics projection produces a model-consistent executed pose; median and maximum observed wall-clock solve times on the evaluation workstation were $28$ and $104$\,ms. These results establish anytime-capable operation on the tested platform, not hard-real-time execution. Closed-loop evaluation covered $150$ WOMD scenarios in Waymax and included reactive and proposal-and-select baselines.

The evaluation also separated log-independent safety and regulatory rules from two rules defined relative to expert replay. W-SQP's mean difference on the sixteen log-independent rules was $-0.04\pp$, with no systematic group-level deficit, whereas the log-coupled difference was $-10.1\pp$. Several individual safety rules nevertheless regressed in high-divergence scenarios (Table~\ref{tab:two_deficits}); the group-level result should not be read as per-rule or per-scenario parity.

For future work, we are interested in nondimensionalizing the optimization surrogates, improving the collision geometry and route construction, evaluating with additional agent and metric models, and profiling the controller on embedded automotive hardware. Hardware-in-the-loop and vehicle experiments must also incorporate perception, estimation, communication, and actuation latency, and combining the tiered-slack controller with certified feasibility or reachability mechanisms remains necessary for a formal safety claim.

\appendix
\setcounter{table}{0}
\setcounter{figure}{0}
\section{Rule Catalogue and Threshold Definitions}
\label{app:catalogue}
This appendix lists the implemented evaluator rulebook. For each rule we give its identifier, priority level, short name, violation condition, and threshold $\tau_r$; a rule is violated at step $k$ when it is applicable and its Boolean predicate $\phi_r(\mathcal{S}_k)$ is true. A global strictness multiplier $\kappa=1.25$ tightens selected thresholds, either raising a required margin or lowering an allowed tolerance by $\kappa$.

\begin{table*}[!t]
\centering
\caption{Complete rule catalogue with predicate and threshold definitions for all $25$ evaluator rules. ``Proxy'' marks the \emph{log-coupled} rules (Def.~\ref{def:logcoupled}) that reference the logged expert trajectory. This table makes the per-rule results (Table~\ref{tab:per_rule_compliance}) reproducible from the paper.}
\label{app:rules}
\scriptsize
\setlength{\tabcolsep}{4pt}
\renewcommand{\arraystretch}{0.92}
\begin{tabular}{c l l p{5.9cm} l}
\toprule
Lvl & ID & Name & Violation condition & Threshold \\
\midrule
10 & rule\_10r0 & VRU collision avoidance & Min corner-to-corner distance to any pedestrian/cyclist $<\tau$ & $0.5\kappa$\,m \\
10 & rule\_10r4 & Cyclist passing distance & Effective lateral clearance to cyclist $<\tau$ & $(1.5+0.05v)\kappa$\,m \\
10 & rule\_10r5 & Bike lane encroachment & Min distance from ego corners to bike-lane points $<\tau$ & $0.3\kappa$\,m \\
\midrule
9  & rule\_9r0  & Vehicle collision avoidance & Min corner-to-corner distance to any vehicle $<\tau$ & $0.5\kappa$\,m \\
9  & rule\_9r1  & Non-traversable surface & Road-edge intrusion $>0$ (any intrusion) & 0 \\
\midrule
8  & rule\_8r0  & Stop sign compliance & $d<5$\,m to stop sign AND speed $>0.5$\,m/s & 0.5\,m/s \\
8  & rule\_8r1  & Crosswalk yield & Near crosswalk AND ped-in-crosswalk AND speed $>2.0$\,m/s & 2.0\,m/s \\
\midrule
7  & rule\_7r0  & Drivable surface & Dist.\ to nearest lane point $>\tau$ & 5.0\,m \\
7  & rule\_7r1  & Traffic light compliance & $d<15$\,m to red TL AND speed $>1.0$\,m/s & 1.0\,m/s \\
7  & rule\_7r9  & Gore area avoidance & Dense road-edge points near ego AND $d_{\min}<1.5$\,m & 1.5\,m \\
\midrule
6  & rule\_6r4  & Elephant racing & Side-by-side vehicle at similar speed AND ego speed $>15$\,m/s & 15\,m/s \\
\midrule
3  & rule\_3r0  & Speed limit & Overspeed $>1/\kappa$\,m/s above proxy limit & $1/\kappa$\,m/s \\
3  & rule\_3r3  & Following headway & THW $<2\kappa$\,s OR TTC $<3\kappa$\,s & $(2\kappa,\,3\kappa)$\,s \\
3  & rule\_3r4  & Lateral clearance & Min lateral clearance to objects $<\tau$ & $1.0\kappa$\,m \\
3  & rule\_3r6  & Lane intrusion & Lateral TTC $<\tau$ & $2.0\kappa$\,s \\
\midrule
2  & rule\_2r0  & Time window compliance & Avg speed over last 2\,s $<0.5$\,m/s after $t>5$\,s & 0.5\,m/s \\
2  & rule\_2r2  & Route adherence (proxy) & Position dev $>5/\kappa$\,m OR yaw dev $>0.5/\kappa$\,rad & $(5/\kappa,\,0.5/\kappa)$ \\
\midrule
1  & rule\_1r0  & Yield to priority & Cross-traffic conflict heuristic (proximity + speed) & --- \\
1  & rule\_1r3  & Turn signal advisory & Advisory-only (yaw-rate proxy; never violated) & --- \\
1  & rule\_1r4  & Lane selection & Advisory placeholder (never violated) & --- \\
1  & rule\_1r9  & Jerk limit & Peak longitudinal jerk $>\tau$ & 5.0\,m/s$^3$ \\
1  & rule\_1r11 & Lateral acceleration & $|v\cdot\dot\psi|>\tau$ & 3.0\,m/s$^2$ \\
\midrule
0  & rule\_0r0  & Path tracking (proxy) & L2 position error to log $>3/\kappa$\,m & $3/\kappa$\,m \\
0  & rule\_0r2  & Longitudinal comfort & Deceleration $>3/\kappa$\,m/s$^2$ (harsh braking) & $3/\kappa$\,m/s$^2$ \\
0  & rule\_0r3  & Lateral comfort & Lateral jerk $>\tau$ & 3.0\,m/s$^3$ \\
\bottomrule
\end{tabular}
\end{table*}

\paragraph{Navigation-alignment proxy rules.} Rules 0r0 and 2r2 measure similarity to the logged trajectory rather than a legal or safety requirement: path tracking (0r0) uses $\|p_{\mathrm{ego}}-p_{\mathrm{log}}\|$ against $3.0/\kappa$, and route adherence (2r2) uses position ($5.0/\kappa$) and yaw ($0.5/\kappa$) deviation. Retained for diagnostics, they nevertheless favor expert replay and must be read accordingly in any aggregate.

\section*{CRediT authorship contribution statement}
\textbf{Hadi Hajieghrary:} Conceptualization, Methodology, Software, Formal analysis, Investigation, Data curation, Validation, Visualization, Writing -- original draft, Writing -- review \& editing.
\textbf{Benedikt Walter:} Software, Investigation, Validation, Writing -- review \& editing.
\textbf{Chaitanya Shinde:} Software, Investigation, Validation, Writing -- review \& editing.
\textbf{Paul Schmitt:} Conceptualization, Methodology, Writing -- review \& editing.
\textbf{Miguel Hurtado:} Resources, Supervision, Project administration, Writing -- review \& editing.

\section*{Declaration of competing interest}
The authors declare that they have no known competing financial interests or personal relationships that could have appeared to influence the work reported in this paper.

\section*{Declaration of generative AI in the writing process}
During the preparation of this work the authors used a large language model to assist with drafting, editing, and figure generation. The authors reviewed and edited all content and take full responsibility for the content of the publication.

\section*{Data availability}
This study is built on the Waymo Open Motion Dataset~\citep{ettinger_womd_iccv_2021}, which is publicly available under its dataset license. The W-SQP controller, the tiered-slack NLP formulation, the closed-loop simulation harness, the $25$-rule evaluator, and the fair (log-independent) scoring reference implementation will be released in a public repository upon publication; scenario indices and configuration files needed to reproduce the reported runs are included.

\IfFileExists{elsarticle-harv.bst}
  {\bibliographystyle{elsarticle-harv}}
  {\bibliographystyle{plainnat}}
\bibliography{References}

@inproceedings{censi_rulebooks_2019,
  title     = {Liability, Ethics, and Culture-Aware Behavior Specification using Rulebooks},
  author    = {Censi, Andrea and Slutsky, Konstantin and Wongpiromsarn, Tichakorn and Yershov, Dmitry and Pendleton, Scott and Fu, James and Frazzoli, Emilio},
  booktitle = {2019 International Conference on Robotics and Automation (ICRA)},
  year      = {2019},
  pages     = {8536--8542},
}

@INPROCEEDINGS{collin_sotif_rulebooks_2021,
  author={Collin, Anne and Bilka, Artur and Pendleton, Scott and Tebbens, Radboud Duintjer},
  booktitle={2020 IEEE Intelligent Vehicles Symposium (IV)}, 
  title={Safety of the Intended Driving Behavior Using Rulebooks}, 
  year={2020},
  volume={},
  number={},
  pages={136-143},
}

@inproceedings{xiao_rulebased_oc_2021,
  title={Rule-based optimal control for autonomous driving},
  author={Xiao, Wei and Mehdipour, Noushin and Collin, Anne and Bin-Nun, Amitai Y and Frazzoli, Emilio and Tebbens, Radboud Duintjer and Belta, Calin},
  booktitle={Proceedings of the ACM/IEEE 12th International Conference on Cyber-Physical Systems},
  pages={143--154},
  year={2021}
}

@inproceedings{maierhofer_interstate_2020,
  title     = {Formalization of Interstate Traffic Rules in Temporal Logic},
  author    = {Maierhofer, Sebastian and Rettinger, Anna-Katharina and Mayer, Eva Charlotte and Althoff, Matthias},
  booktitle = {2020 IEEE Intelligent Vehicles Symposium (IV)},
  year      = {2020},
}

@inproceedings{maierhofer_intersection_2022,
  title     = {Formalization of Intersection Traffic Rules in Temporal Logic},
  author    = {Maierhofer, Sebastian and Moosbrugger, Paul and Althoff, Matthias},
  booktitle = {2022 IEEE Intelligent Vehicles Symposium (IV)},
  year      = {2022},
}

@article{yamaguchi_rtamt_2024,
  title   = {RTAMT: Runtime Robustness Monitors with Application to Signal Temporal Logic},
  author  = {Yamaguchi, Tomoya and Nickovic, Dejan and others},
  journal = {International Journal on Software Tools for Technology Transfer},
  year    = {2024},
}

@inproceedings{halder_lexicographic_stl_2023,
  title     = {Lexicographic Mixed-Integer Motion Planning with STL Constraints},
  author    = {Halder, Patrick and Christ, Fabian and Althoff, Matthias},
  booktitle = {2023 IEEE 26th International Conference on Intelligent Transportation Systems (ITSC)},
  year      = {2023},
}

@article{lin_smt_repair_2024,
  title={Traffic-rule-compliant trajectory repair via satisfiability modulo theories and reachability analysis},
  author={Lin, Yuanfei and Xing, Zekun and Han, Xuyuan and Althoff, Matthias},
  journal={IEEE Transactions on Robotics},
  year={2025},
  publisher={IEEE}
}

@article{ames_cbf_tac_2017,
  author  = {Ames, Aaron D. and Xu, X. and Grizzle, Jessy W. and Tabuada, Paulo},
  title   = {Control Barrier Function Based Quadratic Programs for Safety Critical Systems},
  journal = {IEEE Transactions on Automatic Control},
  year    = {2017},
  volume  = {62},
  number  = {8},
  pages   = {3861--3876}
}

@inproceedings{ames_cbf_ecc_2019,
  author    = {Ames, Aaron D. and Coogan, Samuel and Egerstedt, Magnus and Notomista, Gennaro and Sreenath, Koushil and Tabuada, Paulo},
  title     = {Control Barrier Functions: Theory and Applications},
  booktitle = {European Control Conference (ECC)},
  year      = {2019}
}

@inproceedings{liu_iterative_dhocbf_2023,
  title     = {Iterative Convex Optimization for Model Predictive Control with Discrete-Time High-Order Control Barrier Functions},
  author    = {Liu, Shuo and Zeng, Jun and Sreenath, Koushil and Belta, Calin A.},
  booktitle = {American Control Conference (ACC)},
  year      = {2023},
}

@inproceedings{allamaa_resafe_2024,
  title     = {Real-time MPC with Control Barrier Functions for Autonomous Driving using Safety Enhanced Collocation},
  author    = {Allamaa, Jean Pierre and Patrinos, Panagiotis and Ohtsuka, Toshiyuki and Son, Tong Duy},
  booktitle = {IFAC-PapersOnLine},
  year      = {2024},
}

@article{shalevshwartz_rss_2017,
  title   = {On a Formal Model of Safe and Scalable Self-driving Cars},
  author  = {Shalev-Shwartz, Shai and Shammah, Shaked and Shashua, Amnon},
  journal = {arXiv preprint arXiv:1708.06374},
  year    = {2017},
}

@article{hasuo_rss_2022,
  title   = {Responsibility-Sensitive Safety: an Introduction with an Eye to Logical Foundations and Formalization},
  author  = {Hasuo, Ichiro},
  journal = {arXiv preprint arXiv:2206.03418},
  year    = {2022},
}

@misc{nuplan_2021,
  author       = {Caesar, Holger and Kabzan, Juraj and Tan, Kok Seang and Fong, Whye Kit and Wolff, Eric and Lang, Alex and Fletcher, Luke and Beijbom, Oscar and Omari, Sammy},
  title        = {nuPlan: A Closed-loop ML-based Planning Benchmark for Autonomous Vehicles},
  year         = {2021}
}

@inproceedings{gulino_waymax_2023,
  title     = {Waymax: An Accelerated, Data-Driven Simulator for Large-Scale Autonomous Driving Research},
  author    = {Gulino, Cole and Fu, Justin and Luo, Wenjie and Tucker, George and Bronstein, Eli and Lu, Yiren and Harb, Jean and Pan, Xinlei and Wang, Yan and Chen, Xiangyu and Co-Reyes, John D. and Agarwal, Rishabh and Roelofs, Rebecca and Lu, Yao and Montali, Nico and Mougin, Paul and Yang, Zoey and White, Brandyn and Faust, Aleksandra and McAllister, Rowan and Anguelov, Dragomir and Sapp, Benjamin},
  booktitle = {Advances in Neural Information Processing Systems (NeurIPS) Datasets and Benchmarks Track},
  year      = {2023},
}

@inproceedings{ettinger_womd_iccv_2021,
  title     = {Large Scale Interactive Motion Forecasting for Autonomous Driving: The Waymo Open Motion Dataset},
  author    = {Ettinger, Scott and Cheng, Shuyang and Caine, Benjamin and Liu, Chenxi and Zhao, Hang and Pradhan, Sabeek and Chai, Yuning and Sapp, Benjamin and Qi, Charles and Zhou, Yin and Yang, Zoey and Chouard, Aurelien and Sun, Pei and Ngiam, Jiquan and Vasudevan, Vijay and McCauley, Alex and Shlens, Jonathon and Anguelov, Dragomir},
  booktitle = {IEEE/CVF International Conference on Computer Vision (ICCV)},
  year      = {2021},
}

@misc{waymax_agents_doc,
  title  = {Waymax Agents: {IDM} Route-Following Policy},
  author = {{Waymo Research}},
  year   = {2023},
  howpublished = {\url{https://github.com/waymo-research/waymax}},
  note   = {Waymax documentation, accessed 2026-02-13},
}

@article{peng_nuplan_r_2025,
  title   = {nuPlan-R: A Closed-Loop Planning Benchmark for Autonomous Driving via Reactive Multi-Agent Simulation},
  author  = {Peng, Mingxing and Yao, Ruoyu and Guo, Xusen and Ma, Jun},
  journal = {arXiv preprint arXiv:2511.10403},
  year    = {2025},
}

@article{hagedorn_nuplan_reactive_2025,
  title   = {When Planners Meet Reality: How Learned, Reactive Traffic Agents Shift nuPlan Benchmarks},
  author  = {Hagedorn, Steffen and Donkov, Luka and Distelzweig, Aron and Condurache, Alexandru P.},
  journal = {arXiv preprint arXiv:2510.14677},
  year    = {2025},
}

@inproceedings{dauner_2023_pdm,
  title={Parting with misconceptions about learning-based vehicle motion planning},
  author={Dauner, Daniel and Hallgarten, Marcel and Geiger, Andreas and Chitta, Kashyap},
  booktitle={Conference on Robot Learning},
  pages={1268--1281},
  year={2023},
  organization={PMLR}
}

@misc{plantf_2023,
  author       = {Cheng, Jie and Chen, Yingbing and Mei, Xiaodong and Yang, Bowen and Li, Bo and Liu, Ming},
  title        = {Rethinking Imitation-based Planner for Autonomous Driving},
  howpublished = {arXiv:2309.10443},
  year         = {2023}
}

@misc{pluto_2024,
  author       = {Cheng, Jie and Chen, Yingbing and Chen, Qifeng},
  title        = {PLUTO: Pushing the Limit of Imitation Learning-based Planning for Autonomous Driving},
  howpublished = {arXiv:2404.14327},
  year         = {2024}
}

@article{schwenzer_mpc_review_2021,
  title   = {Review on model predictive control: an engineering perspective},
  author  = {Schwenzer, Michael and Ay, Manuel and Bergs, Thomas and Abel, Dirk},
  journal = {The International Journal of Advanced Manufacturing Technology},
  year    = {2021},
  volume  = {117},
  pages   = {1327--1349},
}

@article{lai_lexicographic_review_2023,
  title   = {Pure and mixed lexicographic-paretian many-objective optimization: state of the art},
  author  = {Lai, Leonardo and Fiaschi, Lorenzo and Cococcioni, Marco and Deb, Kalyanmoy},
  journal = {Natural Computing},
  year    = {2023},
  volume  = {22},
  pages   = {227--242},
  doi     = {10.1007/s11047-022-09911-4}
}

@inproceedings{abernethy_2024_lexopt,
  title     = {Lexicographic Optimization: Algorithms and Stability},
  author    = {Abernethy, Jacob and others},
  booktitle = {Proceedings of Machine Learning Research},
  year      = {2024},
}

@article{tercan_2024_tlo,
  title   = {Thresholded Lexicographic Ordered Multiobjective Optimization},
  author  = {Tercan, A. and others},
  journal = {arXiv preprint arXiv:2408.13493},
  year    = {2024},
}

@article{mavrotas_2009,
  title   = {Effective implementation of the $\epsilon$-constraint method in multi-objective mathematical programming problems},
  author  = {Mavrotas, George},
  journal = {Applied Mathematics and Computation},
  year    = {2009},
  volume  = {213},
  number  = {2},
  pages   = {455--465},
}

@article{casadi_2019,
  title     = {{CasADi}: A Software Framework for Nonlinear Optimization and Optimal Control},
  author    = {Andersson, Joel A. E. and Gillis, Joris and Horn, Greg and Rawlings, James B. and Diehl, Moritz},
  journal   = {Mathematical Programming Computation},
  year      = {2019},
  volume    = {11},
  number    = {1},
  pages     = {1--36},
}

@article{ipopt_2006,
  title   = {On the Implementation of an Interior-Point Filter Line-Search Algorithm for Large-Scale Nonlinear Programming},
  author  = {W{\"a}chter, Andreas and Biegler, Lorenz T.},
  journal = {Mathematical Programming},
  year    = {2006},
  volume  = {106},
  number  = {1},
  pages   = {25--57},
}

@article{benjamini_hochberg_1995,
  title   = {Controlling the False Discovery Rate: A Practical and Powerful Approach to Multiple Testing},
  author  = {Benjamini, Yoav and Hochberg, Yosef},
  journal = {Journal of the Royal Statistical Society: Series B},
  year    = {1995},
  volume  = {57},
  number  = {1},
  pages   = {289--300},
}

@article{manheim_goodhart_2018,
  title   = {Categorizing Variants of Goodhart's Law},
  author  = {Manheim, David and Garrabrant, Scott},
  journal = {arXiv preprint arXiv:1803.04585},
  year    = {2018}
}

@inproceedings{grislain_igdrivsim_2025,
  author    = {Grislain, Clemence and Vuorio, Risto and Lu, Cong and Whiteson, Shimon},
  title     = {{IGDrivSim}: A Benchmark for the Imitation Gap in Autonomous Driving},
  booktitle = {IEEE/RSJ International Conference on Intelligent Robots and Systems (IROS)},
  year      = {2025},
  note      = {arXiv:2411.04653}
}

@misc{ang_clover_2026,
  author       = {Ang, Sining and Yang, Yuguang and Chen, Canyu and Wang, Yan},
  title        = {{CLOVER}: Closed-Loop Value Estimation and Ranking for End-to-End Autonomous Driving Planning},
  howpublished = {arXiv:2605.15120},
  year         = {2026}
}

@inproceedings{dehaan_causal_confusion_2019,
  author    = {de Haan, Pim and Jayaraman, Dinesh and Levine, Sergey},
  title     = {Causal Confusion in Imitation Learning},
  booktitle = {Advances in Neural Information Processing Systems (NeurIPS)},
  year      = {2019}
}

@inproceedings{weihs_imitation_gap_2021,
  author    = {Weihs, Luca and Jain, Unnat and Liu, Iou-Jen and Salvador, Jordi and Lazebnik, Svetlana and Kembhavi, Aniruddha and Schwing, Alexander},
  title     = {Bridging the Imitation Gap by Adaptive Insubordination},
  booktitle = {Advances in Neural Information Processing Systems (NeurIPS)},
  year      = {2021},
  note      = {arXiv:2007.12173}
}

@inproceedings{kerrigan_soft_2000,
  author    = {Kerrigan, Eric C. and Maciejowski, Jan M.},
  title     = {Soft Constraints and Exact Penalty Functions in Model Predictive Control},
  booktitle = {Proc.\ UKACC International Conference on Control (Control 2000)},
  year      = {2000},
  address   = {Cambridge, UK}
}

@article{scokaert_feasibility_1999,
  author    = {Scokaert, Pierre O. M. and Rawlings, James B.},
  title     = {Feasibility Issues in Linear Model Predictive Control},
  journal   = {AIChE Journal},
  volume    = {45},
  number    = {8},
  pages     = {1649--1659},
  year      = {1999},
  publisher = {Wiley}
}

@article{escande_hqp_2014,
  author  = {Escande, Adrien and Mansard, Nicolas and Wieber, Pierre-Brice},
  title   = {Hierarchical Quadratic Programming: Fast Online Humanoid-Robot Motion Generation},
  journal = {The International Journal of Robotics Research},
  volume  = {33},
  number  = {7},
  pages   = {1006--1028},
  year    = {2014},
  publisher = {SAGE}
}

@article{kanoun_taskpriority_2011,
  author  = {Kanoun, Oussama and Lamiraux, Florent and Wieber, Pierre-Brice},
  title   = {Kinematic Control of Redundant Manipulators: Generalizing the Task-Priority Framework to Inequality Task},
  journal = {IEEE Transactions on Robotics},
  volume  = {27},
  number  = {4},
  pages   = {785--792},
  year    = {2011}
}

\end{document}